%% file: iclr2027_conference.tex
\documentclass{article} % For LaTeX2e
\usepackage{iclr2027_conference,times}

\input{math_commands}

\usepackage{hyperref}
\usepackage{url}
\usepackage[capitalise,nameinlink]{cleveref}
\usepackage{amsmath, amsfonts, amssymb, amsthm}
\usepackage{graphicx}
\usepackage{booktabs}
\usepackage{multirow}
\usepackage{multicol}
\usepackage{wrapfig}

\newtheorem{lemma}{Lemma}
\newtheorem{proposition}{Proposition}

\newcommand{\introparagraph}[1]{\textbf{#1.}} % define own new subsection type: noindent, bold (textsc)

\newcommand{\ckale}{\textsc{CKALE}}
\newcommand{\mime}{\textsc{MIME}}

\title{Measuring Learned Equivariance: \\ An Information Theoretic Approach}

\author{Claire Schlesinger, Wolfgang Gatterbauer \& Robin Walters \\
Khoury College of Computer Sciences\\
Northeastern University\\
Boston, MA 02115, USA \\
\texttt{\{schlesinger.e,w.gatterbauer,r.walters\}@northeastern.edu} \\
}

\iclrfinalcopy % Uncomment for camera-ready version, but NOT for submission.
\begin{document}

\maketitle

\begin{abstract}
% \wolf{3-SENTENCE SUMMARY JUST FOR OURSELVES:
% We want to detect whether a network has learned equivariance without first knowing how its hidden features should transform. 
% NMI provides the theoretical criterion, while CKA gives a scalable (``geometric'') practical test for an important class of feature transformations. 
% This makes it possible to study when and where equivariance emerges in ordinary neural networks without first solving the harder problem of recovering their internal transformation rules.}

Learning symmetries through data augmentation offers an alternative to architectural equivariance, but does not guarantee that a network respects those symmetries.
Measuring learned equivariance is especially difficult for intermediate representations, whose group actions are not specified.
In a finite discrete setting, we show that equivariance requires representations of original and transformed inputs to determine one another exactly, yielding a test based on normalized mutual information (NMI).
For high-dimensional activations, we use centered kernel alignment (CKA) as a scalable test, with an exact guarantee for orthogonal feature transformations.
Across architectures and tasks, augmentation usually increases transformation consistency, showing that learned equivariance can be probed directly from representations without recovering the hidden transformation rules that relate them.
\end{abstract}

%------------------------------------------------------------------------------
\section{Introduction}

Equivariance is a powerful inductive bias in fields such as biology, materials science, and image processing, letting networks produce equivalent outputs despite non-canonical inputs.
However, recent work shows that at equal computational cost, non-equivariant networks can outperform structurally equivariant ones when they instead learn equivariance from a variety of transformed inputs \citep{Abramson2024alphafold3, Chai-1-Technical-Report, wohlwend2024boltz1}.
Such networks tend toward equivariance during training but are not constrained to it, so they may not be perfectly equivariant.
\citet{gruver2023liederivativelearnedequivariance} and \citet{moskalev2023ongenuineinvariance} found that models may have generalization error related to unlearned equivariance for inputs further from the training distribution. 
A practical measure of this error is therefore needed to check whether such models generalize correctly across inputs.

Prior work quantifying learned equivariance in model outputs and latent features assumes the group action on input, hidden, and output features is known and specified a priori \citep{gruver2023liederivativelearnedequivariance}.
This limits these metrics, since in practice the group action may be unknown
(especially on the latent space) 
so checking for equivariance requires testing different candidate actions.

We address this with an information-theoretic approach.
Maximal normalized mutual information between a layer's representations of an input and its transformed copy implies a deterministic, invertible map between them, as equivariance requires, without knowing what that map is.
Because mutual information cannot be estimated reliably for continuous activations, we use centered kernel alignment (CKA) \citep{kornblith2019cka} as a proxy, and show that it is an exact test when the feature transformation is orthogonal.

% We address this with an information-theoretic approach based on \emph{centered kernel alignment} (CKA) \citep{kornblith2019cka}.
% Mutual information (MI) measures the dependence between signals passed through a noisy channel such as a neural network, but it cannot be computed without an exact probability distribution.
% \wolf{suggest to restructure this paragraph:
% - first idea to use MI.
% - but complication: can't be measured reliably
% - solution: let's use CKA which is a proxy of MI 
% - not perfect but still works as we show in theory and practice}
% We therefore use CKA, which measures statistical dependence as a proxy for maximum mean discrepancy (MMD) 
% \wolf{what is the role of MMD here. It is mentioned but never explain why this is important.
% Is it that CKA was developed for approximating MMD originally, but then discovered to also help with approximating MI?}
% and, in certain neural-network settings, for mutual information \citep{ni2023learningsparseneuralnetworks}.
% Comparing a layer's outputs on the original and transformed inputs then lets us measure equivariance.

Our contributions are:
\begin{itemize}
    \item We show how learned equivariance can be tested without specifying how hidden features should transform: exact equivariance requires NMI to be maximal between original and transformed representations.

    \item We introduce CKA as a scalable geometric test for high-dimensional activations (comparing their pairwise similarity structure), show that exact equivariance yields maximal CKA under orthogonal feature transformations, and characterize where it departs from the information-theoretic criterion.

    \item Across architectures, tasks, and symmetry groups, we show that data augmentation generally increases learned equivariance in standard networks, and identify where existing equivariance measures succeed or fail.
\end{itemize}

\begin{figure}
    \centering
    \includegraphics[width=\linewidth]{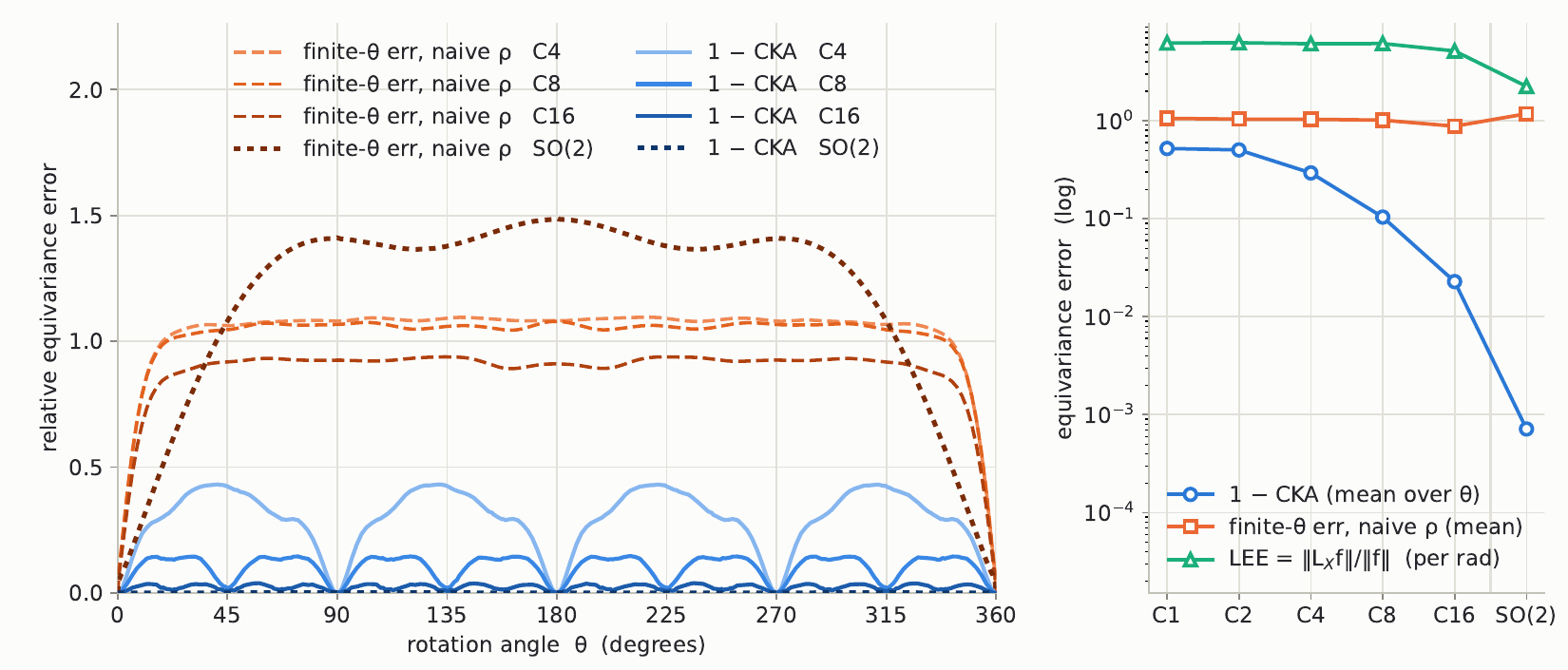}
    \caption{Comparison of equivariance error between $1-\text{CKA}$ and the local equivariance error (LEE) using the assumptions of \citet{gruver2023liederivativelearnedequivariance}. Each model is an exactly equivariant model made with ESCNN \citep{escnn} and equivariance error is computed at every half a degree in the range $[0, 359.5]$. CKA is able to properly detect the zeros at the right intervals where the LEE is unable to detect equivariance for any model, even the SO(2).}
    \label{fig:lee-vs-cka}
\end{figure}

%------------------------------------------------------------------------------
\section{Related Work}

\introparagraph{Measuring and discovering equivariance}
\citet{LencVedaldi2015equivarianceandequivalence} measure equivariance
by fitting feature-space maps for known image transformations,
while \citet{gruver2023liederivativelearnedequivariance} use Lie
derivatives to measure local deviations from specified input and
output actions.
% We instead compare representations directly, without fitting a
% feature transformation or specifying the action on hidden features.
% 
LieGG \citep{moskalev2022liegg} recovers generators of learned
invariances from a gradient-based matrix, while
LieGAN \citep{yang2023generative} discovers distribution-preserving
transformations through adversarial training.
% These methods seek to identify symmetries. 
% We instead take the input
% transformations as given and assess the resulting representations,
% without recovering generators or requiring the features themselves
% to remain invariant.
Our NMI and CKA measurements take the input transformations as given
and compare activations on original and transformed inputs.
They require neither specifying nor fitting the corresponding
transformations of hidden features.
SEIS \citep{lin2026seissubspacebasedequivarianceinvariance} reports
separate equivariance and invariance scores without fitting feature
transformations; we show its equivariance score drops for steerable
networks whose rotations also mix channels.

\introparagraph{Information-theoretic symmetry}
In a grid-world model of shared concept formation,
\citet{moller2023concepts} measure equivariance using mutual
information between concepts assigned to original and transformed
world states.
\citet{charvin2023equivariances} connect equivariance of discrete
probabilistic mappings to optimal joint input-output compression
that preserves mutual information.
Our focus is measuring existing neural representations under
given input transformations, rather than finding an optimal
compression.
For finite discrete representations, we express the necessary
condition through NMI: original and transformed representations
must determine one another exactly.
We then use CKA to compare high-dimensional activations and
investigate where it agrees with or departs from this condition.

\introparagraph{Representation similarity}
SVCCA \citep{raghu2017svcca} compares correlated feature combinations
after discarding low-variance directions, while CKA
\citep{kornblith2019cka} compares patterns of similarity between
examples.
We use linear CKA across layers and during training because it
requires neither selecting feature directions nor fitting an
alignment.
We also evaluate SVCCA and Procrustes error.
Procrustes measures the mismatch remaining after the best rotation
or reflection in feature space \citep{williams2021shape}.
We show that exact equivariance under orthogonal feature
transformations yields maximal CKA, connecting its established
invariance properties to a necessary condition for equivariance.

Similarity scores need not track model behavior
\citep{ding2021grounding,davari2022reliability}, and finite-sample
bias can inflate CKA \citep{murphy2024correcting}.
For equivariance, low CKA need not mean information was lost, and
high CKA does not establish that a particular transformation rule
was followed.
We therefore use CKA to track transformation consistency across
layers and training, with shuffled-pair nulls as a chance baseline
and SVCCA and Procrustes as complementary checks.

\introparagraph{Learning symmetry from data}
Prior work studies how data, augmentation, and architecture affect
learned symmetry \citep{bouchacourt2021grounding,
Bruintjes2023whataffectslearnedequivariance}.
\citet{wang2023mnistfont} analyze how mismatches between symmetry
assumptions and the task or data affect prediction error.
Whereas \citet{Bruintjes2023whataffectslearnedequivariance}
match individual feature maps across channels, we compare
whole representations with CKA.
We use this comparison to examine which layers change with
augmentation and how those changes develop during training,
across image-recognition, image-to-image, and point-cloud tasks.

\section{Background}

We provide background on information theory \citep{shannon1948mathematical}, equivariance, and CKA \citep{kornblith2019cka}, including the Hilbert--Schmidt Independence Criterion (HSIC) \citep{gretton2007hsic} on which CKA is built.

%------------------------------------------------------------------------------
\subsection{Information Theory}

For a discrete random variable $X \sim p(x)$, the Shannon entropy $H(X) = -\sum_x p(x) \log p(x)$ quantifies the information gained by observing $X$.
Given a second variable $Y \sim p(y)$, the conditional entropy is $H(X|Y) = -\sum_y \sum_x p(x,y)\log p(x|y)$.
% \wolf{information gained by observing $X$ if we already know $Y$. if little space, you may want to leave away the def of entropy but give conditional entropy, because that is more related to what comes later.}
% \wolf{I would not say ``we define'' because that is the definition. 
% Maybe to shortly explain for the reader what MI captures.}
The mutual information between the pair of random variables is determined by:
\begin{equation}
    I(X;Y) = H(X) - H(X|Y)
\end{equation}
Mutual information is $0$ if and only if $X$ and $Y$ are independent, i.e., observing $Y$ yields no information about $X$.

% \wolf{same here: better not we define. But ``Normalized MI'' is a canonical way that is restricted to the interval [0,1] (if that is true).}
As mutual information is dependent on the total available information, we utilize normalized mutual information (NMI):
\begin{equation}
    \mathrm{NMI}(X,Y) = \frac{2I(X;Y)}{H(X)+H(Y)}
\end{equation}

%------------------------------------------------------------------------------
\subsection{Equivariance}

A function $f\colon \mathcal{X} \rightarrow \mathcal{Y}$ is $G$-equivariant for a group $G$ 
if $f(\rho_{x}(g)x) = \rho_y(g)f(x)$ for all $g \in G$ and $x \in \mathcal{X}$, where $\rho_x$ and $\rho_y$ are the group actions on the input and output.
Invariance is the special case $\rho_y(g) = \mathrm{id}$, giving $f(\rho_x(g)x) = f(x)$.
As we work with images, we focus on the $C_4$ group of $90^\circ$ rotations, which preserve an image's information.

There are a few ways to measure equivariance error, but one of the simplest ways is to compute the error between the output of the function on the transformed input and the output of the function transformed by the group action $\mathrm{EE} = ||f(\rho_{in}(g)x) - \rho_{out}(g)f(x)||_2^2$. \citet{gruver2023liederivativelearnedequivariance} utilizes equivariance error to measure the per layer equivariance of a model given a infinitesimal transformation of the input. 
In practice, they take $\rho_{out}(g)$ to be trivial for scalar data or a purely spatial rotation for image-like data, which is wrong for steerable equivariant networks \citep{escnn}.

%------------------------------------------------------------------------------
\subsection{Hilbert-Schmidt Independence Criterion}

Given two datasets $\mathbf{X} \in \mathbb{R}^{n \times d_\mathbf{X}}$ and $\mathbf{Y} \in \mathbb{R}^{n \times d_\mathbf{Y}}$ and feature maps $\phi$ and $\theta$ into a Hilbert space, we form kernels $K_\mathbf{X}(x, x') = \langle \phi(x), \phi(x') \rangle$ and $K_\mathbf{Y}(y, y') = \langle \theta(y), \theta(y') \rangle$ that capture nonlinear relationships.
Note that we obtain a linear kernel by choosing $\theta$ and $\phi$ to be the identity function and we obtain an RBF kernel by setting $K_\mathbf{X}$ to be $\exp(-||x-x'||^2/2\sigma^2)$.
With centering matrix $H = \mathbf{I}_n - \frac{1}{n}\mathbf{1}_n\mathbf{1}_n^\top$, the Hilbert-Schmidt Independence Criterion (HSIC) is defined as:
\begin{equation}
    \mathrm{HSIC}(K_\mathbf{X}, K_\mathbf{Y}) = \langle H K_\mathbf{X} H, H K_\mathbf{Y} H \rangle_{F}
\end{equation}

%------------------------------------------------------------------------------
\subsection{Centered Kernel Alignment}

Centered Kernel Alignment (CKA) normalizes HSIC to the range $[0,1]$, removing dependence on the scale of the original data:
\begin{equation}
    \mathrm{CKA}(K_\mathbf{X}, K_\mathbf{Y}) = \frac{\mathrm{HSIC}(K_\mathbf{X},K_\mathbf{Y})}{\sqrt{\mathrm{HSIC}(K_\mathbf{X},K_\mathbf{X)}\mathrm{HSIC}(K_\mathbf{Y},K_\mathbf{Y})}}
\end{equation}
A value of $0$ indicates no dependence captured by the kernel and $1$ a perfect statistical relationship between $\mathbf{X}$ and $\mathbf{Y}$.

%------------------------------------------------------------------------------
\section{Methodology}

\begin{wrapfigure}{r}{0.6\linewidth}
    \centering
    \includegraphics[width=\linewidth]{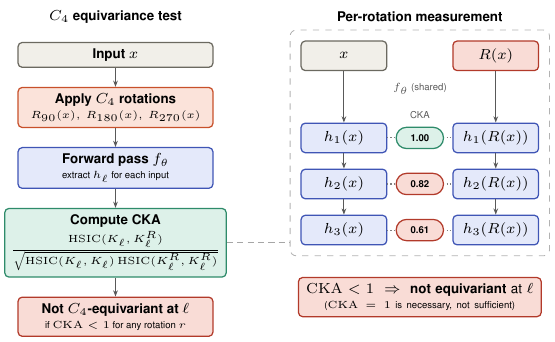}
    \caption{Per-layer \ckale{} test for $C_4$: CKA $<1$ at any rotation rules out exact equivariance at that layer.
    % \wolf{questions about this figure: the right side shows the layers. What is equivariant at I, is that the index of the layer?
    % Now what is idea of the left column? It seems the workflow}
    } %In the example shown above, we use the $C_4$ group.}
    \label{fig:system-diagram}
\end{wrapfigure}

We present \mime{}, a way to utilize mutual information to measure equivariance in discrete cases, and \ckale{}\footnote{Pronounce like kale}, a way to utilize CKA to measure equivariance in real models. We show theoretically that both can operate as an equivariance measure.

%------------------------------------------------------------------------------
\subsection{\mime{}: Mutual Information Measure of Equivariance}

Let $f_\theta$ be a neural network and denote the truncated mapping to the output of layer $\ell$ as $h_\ell$.
For inputs $X$ and group $G$, equivariance cannot be ruled out at layer $\ell$ if for all $g \in G$, $\mathrm{NMI}(h_\ell(X), h_\ell(\rho(g)X)) = 1$.
That is, the layer's outputs on the original and group-transformed data share all of their information, so a deterministic map relates $h_\ell(X)$ and $h_\ell(\rho(g)X)$ in both directions, as equivariance requires.
We call this mutual information measure of equivariance \mime{}.

We show that NMI being 1 is a necessary but not sufficient condition for a function to be equivariant, which justifies its use as a measure.

\begin{proposition}
    Let $g \in G$. If $\mathrm{NMI}(f(X), f(\rho(g)X)) < 1$, then $f$ is not $G$-equivariant at $g$.
\end{proposition}

\begin{proof}
    Proof by contrapositive.
    Let $X$ be a random variable taking values in $\mathcal{X}$ and let $f\colon \mathcal{X} \rightarrow \mathcal{Y}$ be $G$-equivariant, so that $f(\rho_x(g)x) = \rho_y(g)f(x)$ for all $x \in \mathcal{X}$. 
    Define the random variables $Y := f(X)$, and $Y' := f(\rho_x(g)X)$.
    By equivariance, $Y' = \rho_y(g)Y$, and since $\rho_y(g)$ is a group action and hence invertible, $Y = \rho_y(g)^{-1}Y'$. Each of $Y$ and $Y'$ is therefore a deterministic function of the other, so $H(Y \mid Y') = H(Y' \mid Y) = 0$, and by \Cref{lemma:mutual-determinism},
    $\mathrm{NMI}(Y, Y') = \mathrm{NMI}(f(X), f(\rho_x(g)X)) = 1$.
\end{proof}

%------------------------------------------------------------------------------
\subsection{\ckale{}: CKA for Learned Equivariance}

Mutual information requires the full data distribution, which is intractable for a layer's continuous, high-dimensional outputs: recovering the density on $\mathbb{R}^d$ would require evaluating every possible input. 
Instead, to measure learned equivariance in real models, we apply CKA to each layer's output and its rotated copy to test whether the layer acts in an equivariant manner (\Cref{fig:system-diagram}), a method we call \ckale{}: CKA for Learned Equivariance.
We measure $\mathrm{CKA}(h_\ell(x), h_\ell(\rho_{in}(g)x))$ for all $\ell$ in the network $f_\theta$.
For any discrete group, we test every possible input operation when measuring equivariance, but for infinite groups such as $SO(2)$ we subsample to 16 samples to test for equivariance.
We show that similar to NMI, CKA being 1 is a necessary but not sufficient condition for a function to be equivariant which justifies its use as a measure of equivariance.

\begin{proposition}
    Let $g \in G$. Assume $\rho_\mathbf{Y}(g)$ is orthogonal. Also assume orthogonally invariant kernels $K$ that use the same feature map. (For example, linear or RBF kernels).
    If $\mathrm{CKA}(K_{f(\mathbf{X})}, K_{f(\rho_\mathbf{X}(g)\mathbf{X})}) < 1$,  
    the function $f$ is not $G$-equivariant at $g$. 
\end{proposition}

\begin{proof}
    Proof by contrapositive.
    Assume $f\colon \mathbf{X} \rightarrow \mathbf{Y}$ is $G$-equivariant and $\mathbf{X}$ and $\mathbf{Y}$ are sets. 
    Let $\mathbf{Y} := f(\mathbf{X})$ and $\mathbf{Y}' := f(\rho_\mathbf{X}(g)\mathbf{X})$. 
    Since $f$ is $G$-equivariant, there exists an invertible orthogonal map $\rho_\mathbf{Y}(g)$ such that $\mathbf{Y}' =\rho_\mathbf{Y}(g)\mathbf{Y}$. 
    Following our two assumptions, $\rho_\mathbf{Y}(g)$ is orthogonal and our kernel functions $K_\mathbf{Y}$ and $K_{\mathbf{Y}'}$ are orthogonally invariant.
    Therefore $K_\mathbf{Y} = \langle \theta(y), \theta(y') \rangle = \langle \theta(\rho_y(g)y), \theta(\rho_y(g)y') \rangle = K_{\mathbf{Y}'}$.
    By \cref{lemma:extremal-cka-proprotional-gram-matrices}, $\mathrm{CKA}(K_\mathbf{Y}, K_{\mathbf{Y}'}) = 1$.
\end{proof}

%------------------------------------------------------------------------------
\section{Experiments}

We first validate MI as an equivariance measure and compare it with CKA, then show where MI breaks down, and finally compare \ckale{} with existing measures on real models.

%------------------------------------------------------------------------------
\subsection{Mutual Information as an Equivariance Measure}

\begin{table}[t]
    \centering
    \caption{NMI between filtered original and rotated images for the Sobel-X filter (not $C_4$-equivariant) and the Laplacian ($C_4$-equivariant) on a restricted, discretized image set. The Laplacian stays at $1$. Sobel-X is $C_2$-equivariant and anti-commutes, so it reads $1$ at $180^\circ$; a metric assuming the output action matches the input action would report a large error here, but NMI detects the equivariance without being told the action is a sign flip.}
        \footnotesize
    \setlength{\tabcolsep}{2pt} % default is 6pt
    \begin{tabular}{c|cc}
         & \multicolumn{2}{c}{Filter} \\
        Rotation & Sobel-X & Laplacian \\
        \hline
        $90^\circ$ & \textbf{0.1954} & 1.0000 \\
        $180^\circ$ & 1.0000 & 1.0000 \\
        $270^\circ$ & \textbf{0.1954} & 1.0000 \\
        \hline
    \end{tabular}
    \label{tab:mi-proof}
\end{table}

To show that mutual information can measure equivariance, we use a toy dataset of $4 \times 4$ binary images taking values in $\{-1, 1\}$.
Only $2^{16}=65536$ such images exist, so we run every possible image through the chosen filter and use the output rates to compute the probabilities and normalized mutual information. Outputs are binarized at a shared median and coarse-grained by summing each $2\times2$ quadrant. This commutes with the group action, since a $90^\circ$ rotation permutes the quadrants and each $2\times2$ sum is rotation-invariant; \Cref{sec:granularity} explains why it is needed.

As \Cref{tab:mi-proof} shows, NMI falls below $1$ only for Sobel-X at $90^\circ$ and $270^\circ$, since it is only $C_2$-equivariant.

\begin{table}[t]
    \centering
    \caption{CKA versus MI on a discrete set of data points. We apply the Laplacian to the input image, then either keep the output unchanged, scramble the alphabet (relabel outputs while keeping points with the same label together), or merge the alphabet (map multiple distinct outputs to the same value). CKA reads the lossless relabeling as non-equivariant, while MI correctly detects equivariance. A high null indicates the metric may be less informative due to a large alphabet.}
    \label{tab:mi-standard}
    \footnotesize
    \setlength{\tabcolsep}{2pt} % default is 6pt
    \begin{tabular}{llcccccc}
        \toprule
        & & \multicolumn{2}{c}{linear CKA} & \multicolumn{3}{c}{exact MI} \\
        \cmidrule(lr){3-4}\cmidrule(lr){5-7}
        readout & nature & value & null & bits & MI$/H$ & null$/H$ \\
        \midrule
        identity & reference & 1.0000 & 0.0004 & 7.3205 & 1.000 & 0.111 \\
        scramble $p{=}1.00$ & bijective, \textbf{lossless} & 0.0376 & 0.0009 & \textbf{7.3205} & \textbf{1.000} & 0.111 \\
        merge $q{=}0.10$ & \textbf{90\% of alphabet destroyed} & \textbf{0.0431} & 0.0006 & \textbf{3.1574} & \textbf{0.606} & 0.005 \\
        \midrule
        \multicolumn{2}{l}{Sobel-x control (not C4-equivariant)} & 0.0021 & 0.0004 & 1.2676 & 0.195 & 0.031 \\
        \bottomrule
    \end{tabular}
\end{table}

\Cref{tab:mi-standard} shows that relabeling the alphabet, a lossless bijection, drives CKA toward its null because CKA depends on geometry, while MI correctly stays at $1.000$; both correctly flag the lossy merge.
The null column shuffles the pairing between outputs. MI's high null reflects the large alphabet and makes it less trustworthy here, so CKA remains the more practical measure despite missing this form of equivariance.

%----------------------------------------------------------------
\subsection{CKA for Symmetry Probing}

\begin{table}[t]
    \centering
    \caption{Measuring equivariance after a Laplace operation is applied to the discretized images. The transformation is applied so we have a more interesting probability distribution under transformations. Under certain operations applied to the output we see how CKA, MI, and SVCCA either fails or succeeds to detect equivariance. SVCCA is able to detect one more group than CKA, but lacks an unbiased estimator, so we chose to utilize CKA.}
    \label{tab:mi-ladder}
    \footnotesize
    \setlength{\tabcolsep}{2pt} % default is 6pt
    \begin{tabular}{llcccccc}
        \toprule
        readout & nature & MI & null & MI$/H$ & linCKA & rbfCKA & SVCCA \\
        \midrule
        identity          & reference & 7.3205 & 0.8113 & 1.000 & 1.0000 & 1.0000 & 1.0000 \\
        bit-permute       & bijective, isometry & 7.3205 & 0.8113 & 1.000 & 1.0000 & 1.0000 & 1.0000 \\
        \midrule
        ill-cond.\ linear & bijective, linear, \emph{not} isometry & 7.3205 & 0.8113 & 1.000 & 0.3584 & 0.2280 & \textbf{1.0000} \\
        alphabet-permute  & bijective, \emph{not} linear & 7.3205 & 0.8113 & 1.000 & 0.1113 & 0.1543 & \textbf{0.2597} \\
        lossy-half        & \textbf{many-to-one, destroyed} & \textbf{1.8643} & 0.0065 & \textbf{0.505} & 0.4835 & 0.4528 & 0.5042 \\
        \bottomrule
    \end{tabular}
\end{table}

Mutual information is invariant to any bijective reparameterization of the features and falls only when information is genuinely destroyed. Linear CKA is invariant to a strictly smaller set of orthogonal transformations and isotropic scaling. Within that set the two agree exactly: on the identity and bit-permutation readouts of \Cref{tab:mi-ladder}, both of which are orthogonal, CKA reads $1.0000$ and MI reads $7.3205$ bits at MI$/H = 1.000$. Outside it, CKA collapses while MI does not.

Every representation of a compact group is equivalent to an orthogonal one, so the $\rho(g)$ of an architecturally equivariant network can be taken orthogonal, which is CKA's invariance class.
We assume this holds approximately for networks trained with augmentation. Because each comparison is between augmented and baseline models of identical architecture, systematic departures largely cancel.
SVCCA is invariant to all invertible linear maps and survives the ill-conditioned readout, but canonical correlations are biased upward as feature dimension approaches sample count, which is the regime of our batch-size-64 sweep. HSIC admits an unbiased estimator \citep{song2012feature}, which we use to compute CKA throughout \citep{nguyen2021do}.

\begin{table}[t]
    \centering
    \caption{Comparison between different methods on detecting equivariance given an identity and an energy readout. All methods except LEE succeed on the identity readout. Under the energy readout, only mutual information stays exact, while CKA still clearly separates the equivariant kernel from the control.}
   % \caption{The energy readout $\sigma(v) = [Mv \mid (Nv)^2]$ and what each family
   %  reports. Nulls are shuffled controls over 20 permutations, with the fit redone
   %  on the shuffled data for the fitted maps. Separation is the ratio between the
   %  equivariant kernel and the non-equivariant control, oriented so that larger is
   %  better; below $1$ means the metric ranks the control higher. Panel~(a) is the
   %  reference: when $\rho(g)$ is linear the fitted maps are exact. The readout is
   %  injective and MI$/H = 1.000$ in both panels, so nothing is lost and any reported
   %  error belongs to the metric rather than the representation. Exhaustive $2^{16}$
   %  space, quadrant coarse-grained; $H = 7.3205$ bits, control $6.4871$.}
    \label{tab:energy}
    \footnotesize
    \setlength{\tabcolsep}{2pt} % default is 6pt
    \begin{tabular}{llccccc}
        \toprule
        & & \multicolumn{2}{c}{Laplace} & \multicolumn{2}{c}{Sobel-X} & \\
        \cmidrule(lr){3-4}\cmidrule(lr){5-6}
        & metric & value & null & value & null & separation \\
        \midrule
        \multicolumn{7}{l}{\textbf{(a) identity readout} --- reference} \\
        \midrule
        assumes $\rho$ & LEE, $\rho=I$ $\downarrow$ & 1.4933 & 1.4145 & 1.3959 & 1.4150 & \textbf{0.93}$\times$ \\
        \midrule
        \multirow{2}{*}{fits $\rho$}
         & affine fit $\downarrow$ & \textbf{0.0000} & 1.0001 & 1.0836 & 1.0001 & $\infty$ \\
         & Procrustes $\downarrow$ & \textbf{0.0000} & 1.4874 & 1.4884 & 1.4641 & $\infty$ \\
        \midrule
        \multirow{3}{*}{no $\rho$}
         & linear CKA $\uparrow$ & 1.0000 & 0.0020 & 0.0023 & 0.0007 & 435$\times$ \\
         & RBF CKA $\uparrow$ & 1.0000 & 0.0030 & 0.0044 & 0.0014 & 227$\times$ \\
         & SVCCA $\uparrow$ & 1.0000 & 0.0387 & 0.0936 & 0.0283 & 10.7$\times$ \\
        \midrule
         & MI, exact (bits) $\uparrow$ & 7.3205 & 0.8100 & 1.2676 & 0.2027 & 5.8$\times$ \\
         & \quad MI$/H$ \ (null$/H$) & 1.000 & 0.111 & 0.195 & 0.031 & --- \\
        \midrule
        \midrule
        \multicolumn{7}{l}{\textbf{(b) \textsc{energy} readout} $\sigma(v)=[Mv \mid (Nv)^2]$} \\
        \midrule
        assumes $\rho$ & LEE, $\rho=I$ $\downarrow$ & 1.3640 & 1.4140 & 1.3958 & 1.4139 & \textbf{1.02}$\times$ \\
        \midrule
        \multirow{2}{*}{fits $\rho$}
         & affine fit $\downarrow$ & 0.6723 & 1.0002 & 1.0438 & 1.0002 & 1.55$\times$ \\
         & Procrustes $\downarrow$ & 0.7225 & 1.4749 & 1.4511 & 1.4422 & 2.01$\times$ \\
        \midrule
        \multirow{3}{*}{no $\rho$}
         & linear CKA $\uparrow$ & 0.7379 & 0.0029 & 0.0071 & 0.0015 & \textbf{104}$\times$ \\
         & RBF CKA $\uparrow$ & 0.8634 & 0.0042 & 0.0091 & 0.0023 & \textbf{95}$\times$ \\
         & SVCCA $\uparrow$ & 0.7802 & 0.0459 & 0.0876 & 0.0394 & 8.9$\times$ \\
        \midrule
         & MI, exact (bits) $\uparrow$ & \textbf{7.3205} & 0.8100 & 1.2676 & 0.2027 & 5.8$\times$ \\
         & \quad MI$/H$ \ (null$/H$) & \textbf{1.000} & 0.111 & 0.195 & 0.031 & --- \\
        \bottomrule
    \end{tabular}
\end{table}

Squared magnitudes are common in real networks (fiber norms in steerable CNNs, attention scores, second-moment features), and they break linear transformation behavior: if $v(gx) = Pv(x)$, the energy block $(NPv)^2$ contains cross terms $v_jv_k$, so no matrix $\rho(g)$ exists to assume or fit.
We test this with the injective readout $\sigma(v) = [Mv \mid (Nv)^2]$.
In \Cref{tab:energy}, MI$/H$ stays at $1.000$, so any reported error belongs to the metric rather than the representation.
LEE sits at its shuffled null ($1.02\times$ separation) and detects nothing.
The fitted maps still rank the kernels correctly, but they leave residuals near $0.7$ on an exactly equivariant representation, which is a calibration failure that would read as non-equivariance in practice.
CKA and MI both keep working, with linear CKA at $0.7379$ against $0.0071$ on the control.

%------------------------------------------------------------------------------
\subsection{Mutual Information Precision Problem}
\label{sec:granularity}

Higher precision enlarges the alphabet, and the shuffled null rises with it. Our earlier MI experiments used coarse-graining to keep the null low, but choosing a granularity is difficult in practice, so we repeat them at full 16-bit precision.

\begin{table}[t]
    \centering
    \caption{Changing the granularity of what mutual information sees from 4 bits to the full 16 bits shows that the null reading rises with it due to the larger alphabet, making mutual information a much worse measure the more unique outputs there are. In a continuous space with infinite outputs, mutual information is not only intractable to compute, but also loses its ability to distinguish real equivariance from scrambles.}
    \label{tab:mi-granularity}
    \footnotesize
    \setlength{\tabcolsep}{2pt} % default is 6pt
    \begin{tabular}{lrrrrrrr}
        \toprule
        granularity & symbols & samp./sym & MI & $H$ & MI$/H$ & null & margin \\
        \midrule
        \multicolumn{8}{l}{\emph{no convolution} --- raw pixels} \\
        \quad full 16-bit & 65{,}536 & 1.00 & 16.0000 & 16.0000 & 1.000 & 16.0000 & \textbf{0.0000} \\
        \midrule
        \multicolumn{8}{l}{\textbf{Laplacian} --- exactly C4-equivariant} \\
        \quad full 16-bit & 45{,}707 & 1.43 & 15.1834 & 15.1834 & 1.000 & \textbf{14.3668} & \textbf{0.8165} \\
        \quad quadrant    & 486 & 134.8 & 7.3205 & 7.3205 & 1.000 & 0.8157 & \textbf{6.5048} \\
        \midrule
        \multicolumn{8}{l}{\textbf{Sobel-x} --- not C4-equivariant (control)} \\
        \quad full 16-bit & 1{,}225 & 53.5 & 4.6122 & 9.0427 & 0.510 & 2.6101 & \textbf{2.0020} \\
        \quad quadrant    & 139 & 471.5 & 1.2676 & 6.4871 & 0.195 & 0.2035 & 1.0641 \\
        \bottomrule
    \end{tabular}
\end{table}

\Cref{tab:mi-granularity} shows that by increasing to the full 16-bit resolution, the trustworthiness of the MI measure falls because the null measurement rises.

%------------------------------------------------------------------------------
\subsection{\ckale{} as an Assumption Free Equivariance Measure}

Other measures require assumptions about the group action. CKA's assumptions (orthogonal transformations or isotropic scaling) are much weaker. We compare measures on exactly equivariant models and on models trained with and without augmentation.

\begin{table}[t]
\centering
\caption{Three measures on exactly-SO(2) point clouds (three seeds, 300 epochs,
2048 samples for LieGG; lower is more equivariant). LEE ranks the provably
equivariant model last and an untrained network first; LieGG and \ckale{}
recover the correct order. The $\lVert\nabla\rVert$ column is the mechanism:
LEE $=\lVert L_X f\rVert/\lVert f\rVert$ is invariant to uniformly rescaling the
features but not to their being nearly constant, where the numerator vanishes
and the denominator does not. The untrained network's input gradient is over
$700\times$ smaller than the others'. LieGG normalizes by its polarization
matrix's largest singular value, where rescaling cancels, and is unaffected.}
\label{tab:lee-reversal}
\footnotesize
\setlength{\tabcolsep}{2pt} % default is 6pt
\begin{tabular}{lcccrrr}
\toprule
& \multicolumn{3}{c}{LieGG} & & & \\
\cmidrule(lr){2-4}
model & sym.\ var. & null dim & align & $\lVert\nabla\rVert$ & LEE & $1-$CKA \\
\midrule
Equivariant SO(2)    & \textbf{0.00000} & \textbf{1.0} & \textbf{1.0000} & 5.9228 & 0.9999 & \textbf{0.000000} \\
MLP, augmented       & 0.05760 & 0.0 & 0.0000 & 5.9187 & 0.3702 & 0.008651 \\
MLP, no augmentation & 0.06845 & 0.0 & 0.0000 & 6.1034 & 0.6413 & 0.042468 \\
MLP, untrained       & 0.16950 & 0.0 & 0.0000 & \textbf{0.0083} & \textbf{0.2518} & 0.460116 \\
\bottomrule
\end{tabular}
\end{table}

\Cref{tab:lee-reversal} and \cref{fig:lee-vs-cka} show that LEE's assumed output action leads it to report high error where there should be zeros, while CKA and LieGG identify the correct behavior.

\begin{table}[t]
\centering
\caption{LieGG makes an assumption that the input space has to be equivariant and the output space has to be invariant to rotation. Purposely building a model that has an equivariant output space shows that LieGG is unable to detect it.}
\label{tab:liegg-equivariant-output}
\footnotesize
\setlength{\tabcolsep}{2pt} % default is 6pt
\begin{tabular}{llccccc}
\toprule
& & \multicolumn{4}{c}{LieGG} & \\
\cmidrule(lr){3-6}
readout & $m$ & sym.\ var.\ $\downarrow$ & null dim & align $\uparrow$
        & $A_{\rm true}/$rand $\downarrow$ & $1-$CKA $\downarrow$ \\
\midrule
invariant scalar, \ \texttt{head(\textbar\textbar h\textbar\textbar)} & 1
  & 0.00003 & \textbf{1.0} & \textbf{1.0000} & 0.00008 & 0.000000 \\
invariant vector, \ \texttt{\textbar\textbar h\textbar\textbar} & 64
  & 0.00022 & \textbf{1.0} & \textbf{1.0000} & 0.00007 & 0.000033 \\
\textbf{equivariant}, \ \texttt{h} & 128
  & \textbf{0.43667} & \textbf{0.0} & \textbf{0.0000} & \textbf{0.87357}
  & \textbf{0.000000} \\
\bottomrule
\end{tabular}
\end{table}

However, LieGG assumes an invariant output, so it misses equivariance in a model with vector outputs (\cref{tab:liegg-equivariant-output}).

\begin{table}[t]
\centering
\caption{Comparison on SEIS equivariance measure between spatial and channel re-encoding. SEIS is able to detect spatial re-encodings but not channel re-encodings and is therefore unable to detect equivariance of steerable models.}
\label{tab:seis-mechanism}
\begin{tabular}{lccl}
\toprule
transformation of the features & $\mathcal{S}_{\rm equiv}$ & $\mathcal{S}_{\rm inv}$ & SEIS reads it as \\
\midrule
identity (reference) & 1.0000 & 1.0000 & --- \\
\midrule
\textbf{spatial re-encoding} \ rotate the grid & \textbf{1.0000} & 0.1311 & \emph{re-encoded} \ \checkmark \\
\textbf{channel re-encoding} \ the true $\rho(g)$ & \textbf{0.2238} & 0.0941 & \emph{destroyed} \ $\times$ \\
\midrule
disjoint images (reference) & 0.1361 & 0.0767 & --- \\
\bottomrule
\end{tabular}
\end{table}

SEIS has a similar problem to the Lie derivative method where it is unable to detect the equivariance of steerable networks. We compare SEIS between a network that does spatial re-encoding and the steerable network which does spatial and channel re-encodings. \Cref{tab:seis-mechanism} shows that SEIS is able to detect equivariance only when it changes the space of the features, but not the channel of the features. SEIS is reported separately because its construction requires an unpooled spatial grid: it cannot be evaluated on the point-cloud experiments, which have no spatial axis, nor on the pooled comparison, where the axis it operates on has been averaged away.

%------------------------------------------------------------------------------
\subsection{Robustness of \ckale{} Across Many Domains}
\label{sec:robustness}

To analyze the performance of CKA on real life models and datasets, we used 6 unique tasks: image recognition, image recolorization, fluid flow prediction on a grid, fluid flow prediction on a point cloud, stress tensor prediction on a grid, and stress tensor prediction on a point cloud. 
For image recognition, we used CIFAR-10 \citep{krizhevsky2009learning} and MNIST-Font \citep{wang2023mnistfont} datasets. 
For colorization we also used CIFAR-10 and STL-10 \citep{coates2011stl10}.
For the fluid flow and stress datasets we generated our own datasets. 
To generate the fluid flow dataset we generate a bunch of random points with a velocity and temperature in a vertical 2 dimensional slice of space with warm blobs placed lower in the field and cold blobs placed higher in the field. We walk the simulation forward 5 steps and use that as the final result.
To generate the stress dataset we randomly generate the stress potential field. To produce symmetry breaking effects we apply a stiffness and a directional shear coupling along the x direction.
For any grid-like dataset we apply $C_4$ transformations during training the augmented models and for any point cloud dataset we apply $SO(2)$ transformations so the model learns $SO(2)$ equivariance.

For the grid-like datasets we use 4 unique models, a convolutional neural network (CNN), a UNet \citep{ronneberger2015unet}, a vision transformer (ViT) \citep{dosovitskiy2021vit} and a perceptron. 
For the point cloud datasets we again use a perceptron, but also a multilayer perceptron, a PointNet \citep{qi2017pointnet} and a SetTransformer \citep{lee2019settransformer}. None of these models are equivariant which is necessary to see how they may learn equivariance when presented with augmented data.

\begin{figure}
    \centering
    \includegraphics[width=\linewidth]{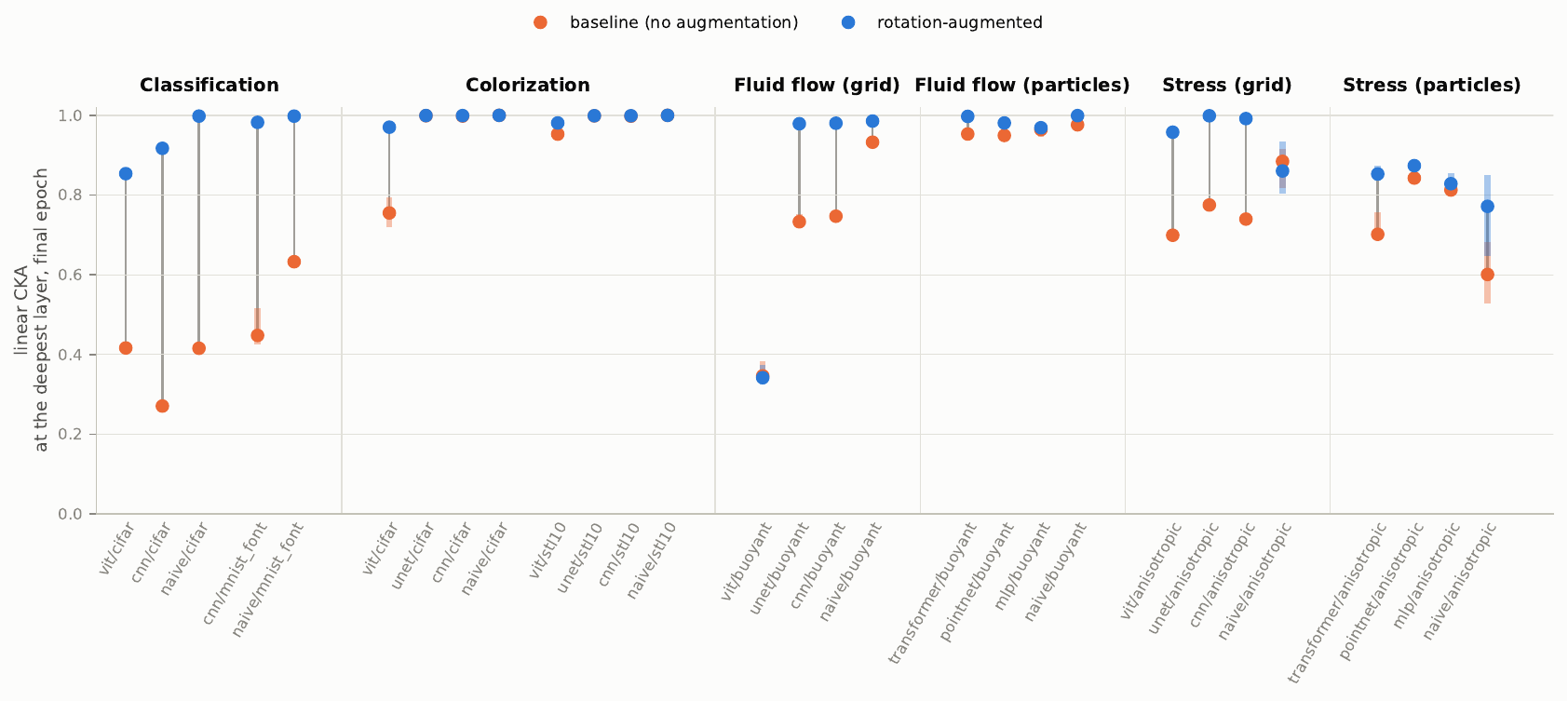}
    \caption{Final-layer CKA across models and datasets for $C_4$ and $SO(2)$ tasks (``naive'' denotes the perceptron). Augmentation increases learned equivariance on most tasks but not colorization, likely because colorization relies more on texture than position. Error bars show the 25th and 75th percentiles.}
    \label{fig:sweep-effects}
\end{figure}

For each dataset and model, we run 10 seeds. We measure the equivariance error against the test dataset with 90 degree rotations for grid like data and 21.17 degree rotations for pointlike data (16 non-identity angles). \Cref{fig:sweep-effects} shows that for most tasks, augmentation improves last-layer equivariance.

%------------------------------------------------------------------------------
\section{Conclusion}

We showed that normalized mutual information gives a necessary condition for equivariance that requires no knowledge of the feature-space action, and that it is reliable in small discrete settings. Because MI degrades with large alphabets and is intractable for continuous activations, we introduced \ckale{}, which is exact under orthogonal feature transformations and scales across tasks and architectures without discovering or assuming the group action.

\paragraph{Limitations and Future Work} Mutual information cannot be computed exactly in continuous spaces. CKA is limited to orthogonal transformations, though it is more resilient to random scrambles and covers the transformation class that steerable equivariant networks use. We also subsample infinite groups. Future work could avoid this subsampling and explore kernels that capture non-orthogonal transformations.

\subsection*{AI use statement}

In this work, we used generative AI tools for creating code to generate synthetic datasets, to assist in the writing of mathematical claims, critical ingredients for those claims, and the construction of proofs, to design and provide feedback on research methodology, and to implement certain experiments.
We have not used generative AI tools to develop theoretical models or conceptual frameworks, propose or refine hypotheses, or interpret results,
and assistance with translation and with cleaning and reformatting datasets is not applicable to this work.
Additionally, we used generative AI tools to create scientific figures, suggest experimental parameters, assist with drafting and editing papers, discover related works, and write code. We have reviewed all AI-assisted work. We verified all references from LLMs to ensure they were correct. One author checked the code for correctness and logical results. All authors ensured any information written on the paper was correct. We take responsibility for the final content of this work,
including text, claims or artifacts produced with the aid of generative AI.

% \subsection*{Ethics statement}

% (This section is \textbf{recommended} and does not count toward the page limit.)

% If authors feel that their paper submission raises questions regarding the Code
% of Ethics, they are encouraged to include a paragraph of Ethics Statement (at
% the end of the main text before references) to address potential concerns where
% appropriate. Topics include, but are not limited to, studies that involve human
% subjects, practices to data set releases, potentially harmful insights,
% methodologies and applications, potential conflicts of interest and sponsorship,
% discrimination/bias/fairness concerns, privacy and security issues, legal
% compliance, and research integrity issues (e.g., IRB, documentation, research
% ethics). This statement should not be more than 1 page.

\subsection*{Reproducibility statement}

All code used to run the experiments in our project was provided in the supplementary materials. 
Additionally, we supply descriptions of the dataset creation and experiments we ran in \cref{sec:robustness}.


% \section{Acknowledgments}

% This work was completed in part using the Explorer Cluster, supported by Northeastern University's Research Computing team. 
% Additionally, Claire Schlesinger is supported by the NSF Graduate Research Fellowship. 
% Robin Walters would like to acknowledge support from NSF Grants 2442658 and 2134178. 
% This work is supported by the National Science Foundation under Cooperative Agreement PHY-2019786 (The NSF AI Institute for Artificial Intelligence and Fundamental Interactions, \url{http://iaifi.org/}).

\bibliography{iclr2027_conference}
\bibliographystyle{iclr2027_conference}

\appendix
\section{Lemmas}

\begin{lemma}
    $\mathrm{NMI}(X, Y) = 1 \iff H(X|Y) = H(Y|X) = 0$.
    \label{lemma:mutual-determinism}
\end{lemma}

\begin{proof}
    By definition of mutual information, $I(X;Y) = H(X) - H(X|Y) = H(Y) - H(Y|X)$, therefore $2I(X;Y) = H(X) + H(Y) - (H(X|Y) + H(Y|X))$. 
    
    If $\mathrm{NMI}(X, Y) = 1$, then by definition $\frac{2I(X;Y)}{H(X) + H(Y)} = \frac{H(X) + H(Y) - (H(X|Y) + H(Y|X))}{H(X) + H(Y)} = 1 - \frac{H(X|Y) + H(Y|X)}{H(X) + H(Y)} = 1$, so $ \frac{H(X|Y) + H(Y|X)}{H(X) + H(Y)} = 0$ which requires $H(X|Y) + H(Y|X) = 0$

    If $H(X|Y) = H(Y|X) = 0$, then $2I(X;Y) = H(X) + H(Y) - (H(X|Y) + H(Y|X)) = H(X) + H(Y)$. Therefore $\mathrm{NMI}(X, Y)=1$.
\end{proof}

\begin{lemma}
\label{lemma:extremal-cka-proportional-gram-matrices}
    Under the nondegeneracy assumption $\tilde{X}\neq 0$ and $\tilde{Y}\neq 0$,
    \[
    \mathrm{CKA}(K_X, K_Y) = 1 \iff \tilde{X} = c\,\tilde{Y} \text{ for some } c > 0,
    \]
    where $\tilde{X} = HK_XH$ and $\tilde{Y} = HK_YH$ are the centered Gram matrices.
    \label{lemma:extremal-cka-proprotional-gram-matrices}
\end{lemma}

\begin{proof}
    Vectorize the centered Gram matrices and treat them as elements of
    $\mathbb{R}^{n \times n}$ with the Frobenius inner product
    $\langle A, B \rangle_F = \operatorname{tr}(A^{\top}B)$. By definition, CKA is the
    cosine of the angle between them:
    $\mathrm{CKA}(K_\mathbf{X}, K_\mathbf{Y}) = \dfrac{\langle \tilde{X}, \tilde{Y} \rangle_F}{\|\tilde{X}\|_F\,\|\tilde{Y}\|_F}$.
    
    By the Cauchy--Schwarz inequality,
    $|\langle \tilde{X}, \tilde{Y} \rangle_F| \leq \|\tilde{X}\|_F\,\|\tilde{Y}\|_F$,
    so $\mathrm{CKA} \in [-1, 1]$, with equality $|\mathrm{CKA}| = 1$ iff $\tilde{X}$
    and $\tilde{Y}$ are linearly dependent, i.e.\ $\tilde{X} = c\,\tilde{Y}$ for some
    $c \neq 0$ (nonzero by nondegeneracy).
    
    Since $K_\mathbf{X}, K_\mathbf{Y}$ are Gram matrices they are positive semidefinite, and centering
    preserves this, so
    $\langle \tilde{X}, \tilde{Y} \rangle_F = \operatorname{tr}(\tilde{X}\tilde{Y})
    = \operatorname{tr}\!\big(\tilde{X}^{1/2}\tilde{Y}\,\tilde{X}^{1/2}\big) \geq 0$.
    Hence the numerator is nonnegative and $\mathrm{CKA} \in [0, 1]$.
    
    Combining the two: $\mathrm{CKA} = 1$ forces the Cauchy--Schwarz equality
    $\tilde{X} = c\,\tilde{Y}$ together with $\langle \tilde{X}, \tilde{Y} \rangle_F \ge 0$.
    Substituting gives $\langle \tilde{X}, \tilde{Y} \rangle_F = c\,\|\tilde{Y}\|_F^2$,
    which is nonnegative only when $c > 0$. Conversely, if $\tilde{X} = c\,\tilde{Y}$
    with $c > 0$, direct substitution yields $\mathrm{CKA} = 1$.
\end{proof}

\section{CKA and Aliasing}
\label{app:cka-aliasing}

\subsection{A resampling floor for the continuous sweep}

The sweep of \Cref{fig:lee-vs-cka} compares $f(x)$ against
$f(\mathrm{rot}(\theta)x)$ at arbitrary $\theta$, and a non-lattice rotation of a
pixel grid is not a permutation. It requires interpolation, which low-pass
filters the image and discards content at the corners, so $\mathrm{rot}(\theta)x$
is not ``$x$, rotated'' but a rotated and slightly damaged $x$. The rising error
away from the group elements could therefore reflect the probe damaging its own
input rather than the model failing to be equivariant.

That possibility can be measured directly. Linear CKA is invariant to orthogonal
maps and a quarter turn is a permutation of pixels, so
$1-\mathrm{CKA}(x, \mathrm{rot}(\theta)x)$ evaluated on \emph{raw pixels with no
network} is exactly zero on the lattice and positive off it. It isolates the
rotation operator's departure from being an isometry, with the model removed.

\begin{table}[h]
\centering
\caption{The resampling floor. \textbf{(a)} A C16 steerable conv at its own
sixteen group elements: the input floor is $1-\mathrm{CKA}$ between raw pixels
and their own rotation, the model column the same quantity on the conv's
features. \textbf{(b)} Both, by grid resolution. All entries are
$1-\mathrm{CKA}$; 128 MNIST images, $29^2$ grid in panel (a).}
\label{tab:resampling-floor}
\begin{tabular}{llccc}
\toprule
\multicolumn{5}{l}{\textbf{(a)} at the C16 model's own group elements} \\
\midrule
rotation & on lattice? & input floor & model & model$/$floor \\
\midrule
$0^\circ, 90^\circ, 180^\circ, 270^\circ$ & yes & \textbf{0.00000} & \textbf{0.00000} & --- \\
$22.5^\circ$ and its 3 repeats & no & 0.00011 & 0.00032 & 2.83 \\
$67.5^\circ$ and its 3 repeats & no & 0.00011 & 0.00026 & 2.30 \\
$45^\circ$ and its 3 repeats & no & 0.00013 & 0.00054 & 4.10 \\
\midrule
off-lattice mean & & 0.00012 & 0.00037 & 3.08 \\
\bottomrule
\end{tabular}

\vspace{1em}

\begin{tabular}{lccccc}
\toprule
\multicolumn{6}{l}{\textbf{(b)} the same quantities, by grid resolution} \\
\midrule
& \multicolumn{2}{c}{$22.5^\circ$} & \multicolumn{2}{c}{$45.0^\circ$} & $90^\circ$ \\
\cmidrule(lr){2-3}\cmidrule(lr){4-5}
grid & floor & model & floor & model & (both) \\
\midrule
$29^2$ & 0.00011 & 0.00032 & 0.00013 & 0.00054 & 0.00000 \\
$57^2$ & 0.00000 & 0.00004 & 0.00000 & 0.00005 & 0.00000 \\
$113^2$ & 0.00000 & 0.00000 & 0.00000 & 0.00000 & 0.00000 \\
$225^2$ & 0.00000 & 0.00000 & 0.00000 & 0.00000 & 0.00000 \\
\bottomrule
\end{tabular}
\end{table}

\Cref{tab:resampling-floor} reports it. The floor behaves as the argument
requires: exactly zero at the four quarter turns, $0.00011$--$0.00013$
elsewhere. The model's error sits about $3\times$ above it, so the convolution
amplifies the input damage rather than merely inheriting it. But both
quantities are of order $10^{-4}$, and both fall to zero by a $113^2$ grid, so
the entire off-lattice signal at these angles is a sampling artifact rather
than a property of the model.

The peaks in the continuous sweep run from $0.04$ for C16 to $0.51$ for C1,
between two and three orders of magnitude above this floor, so the group
structure recovered there is genuine signal. We report the negative because the
confound was plausible and the floor had not previously been measured at the
input, only at the model.

\subsection{CKA does not detect aliasing}

The same experiment read the other way exposes a limitation. Interpolation
genuinely destroys information---it is irreversible, and the representation-aware
equivariance error at these angles is two to three orders of magnitude larger
than what CKA reports. Downsampling the feature map at export makes the point
without any rotation at all.

\begin{table}[h]
\centering
\caption{Downsampling a feature map at export aliases it. Equivariance error is
$\lVert f(gx)-\rho(g)f(x)\rVert / \lVert\rho(g)f(x)\rVert$ with $\rho(g)$ read
off escnn's declared \texttt{FieldType}, so it is ground truth rather than a
competing estimate. The null is the shuffled control: the pairing between the
clean and rotated branches is destroyed while marginals, dimensionality and
sample count are held fixed. $\theta = 2.0$ rad ($114.6^\circ$), 256 MNIST
images.}
\label{tab:cka-blind}
\begin{tabular}{lccc}
\toprule
condition & equivariance error $\downarrow$ & linear CKA $\uparrow$ & CKA null \\
\midrule
grid-aligned reference ($\theta = \pi/2$) & \textbf{0.00000} & \textbf{1.0000} & 0.0729 \\
\midrule
full resolution & 0.13429 & 0.9994 & 0.0722 \\
subsample & \textbf{0.21735} & 0.9995 & 0.0727 \\
blur, then subsample & 0.15250 & 0.9989 & 0.0432 \\
\bottomrule
\end{tabular}
\end{table}

Naive subsampling inflates the equivariance error by 62\%, from $0.13429$ to
$0.21735$, and prefiltering before subsampling brings it back to $0.15250$---a
loss identified by mechanism and then repaired. Linear CKA reads ${\sim}0.999$
through all three and registers neither the damage nor the fix. The baseline is
not zero because $\theta = 2.0$ rad is off-lattice and a C$_N$ model on a square
grid is exact only at quarter turns, so $0.13429$ is the interpolation floor
above which the 62\% is measured.

The null makes the failure precise. A reading of $0.999$ could mean either that
the two feature sets really are nearly identical or that the metric returns
$0.999$ on anything---the second being the degenerate regime RBF CKA enters at
small bandwidth, where signal and shuffled control both reach $1.0000$.
Row-shuffling gives ${\sim}0.07$ here, so CKA sits roughly $14\times$ above its
own floor and the readings are real. The failure is one of \emph{range} rather
than of validity: across the whole experiment CKA spans $1.0000$ to $0.9980$
while the quantity it tracks runs from exactly zero to $0.236$.

It is not, then, that lost data degrades CKA's resolution. It is that CKA barely
notices data has been lost at all, because interpolation perturbs every feature
slightly while leaving their pairwise similarities almost intact.

\section{Progression of Models' Learned Equivariance}
\label{app:progression-learned-equivariance}

Using the same setup as described in \cref{sec:robustness}, we showcase each layer of every model across all epochs to show how learned equivariance evolves over time. 

\begin{figure}[h]
    \centering
    \includegraphics[width=\linewidth]{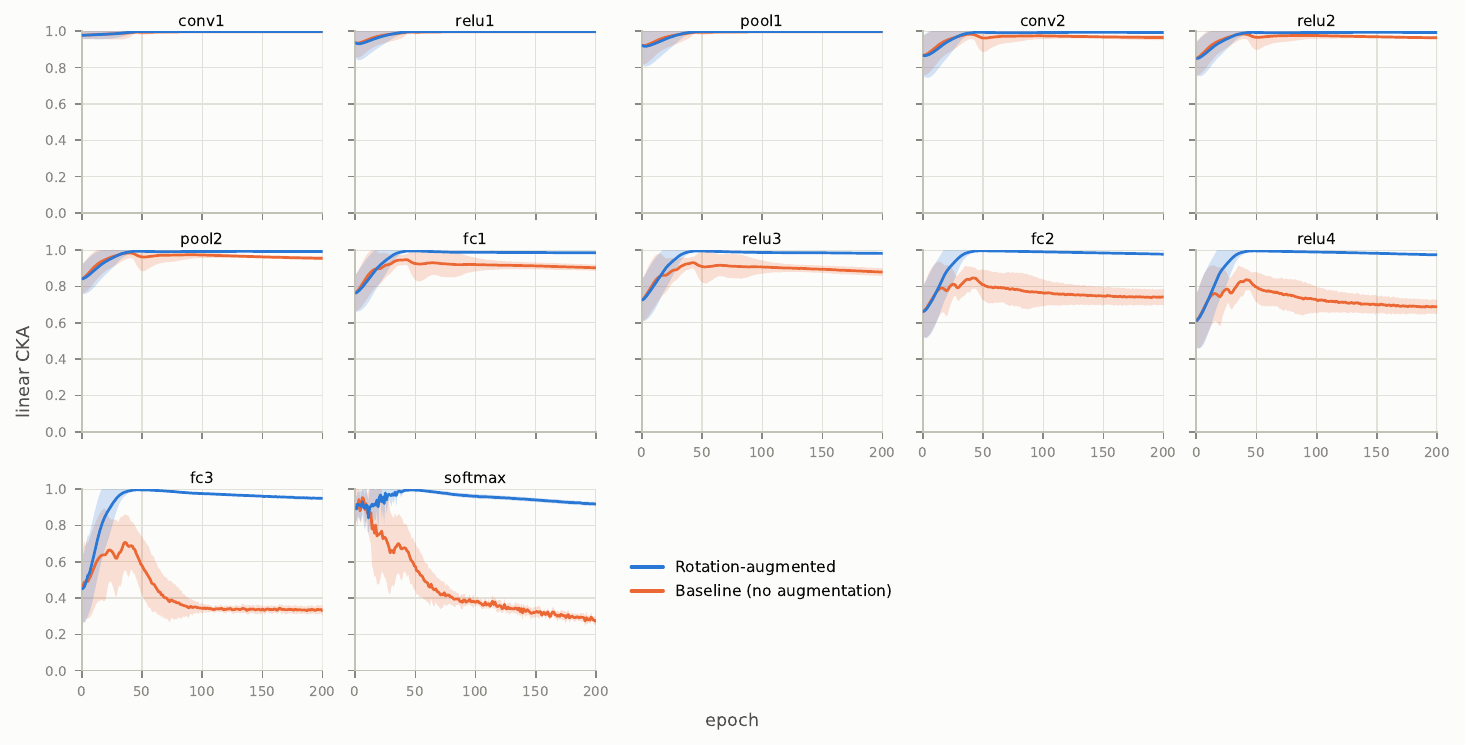}
    \caption{Full view of a CNN progression on learning equivariance on CIFAR-10 classification.}
    \label{fig:classification_cnn_cifar}
\end{figure}

\begin{figure}[h]
    \centering
    \includegraphics[width=\linewidth]{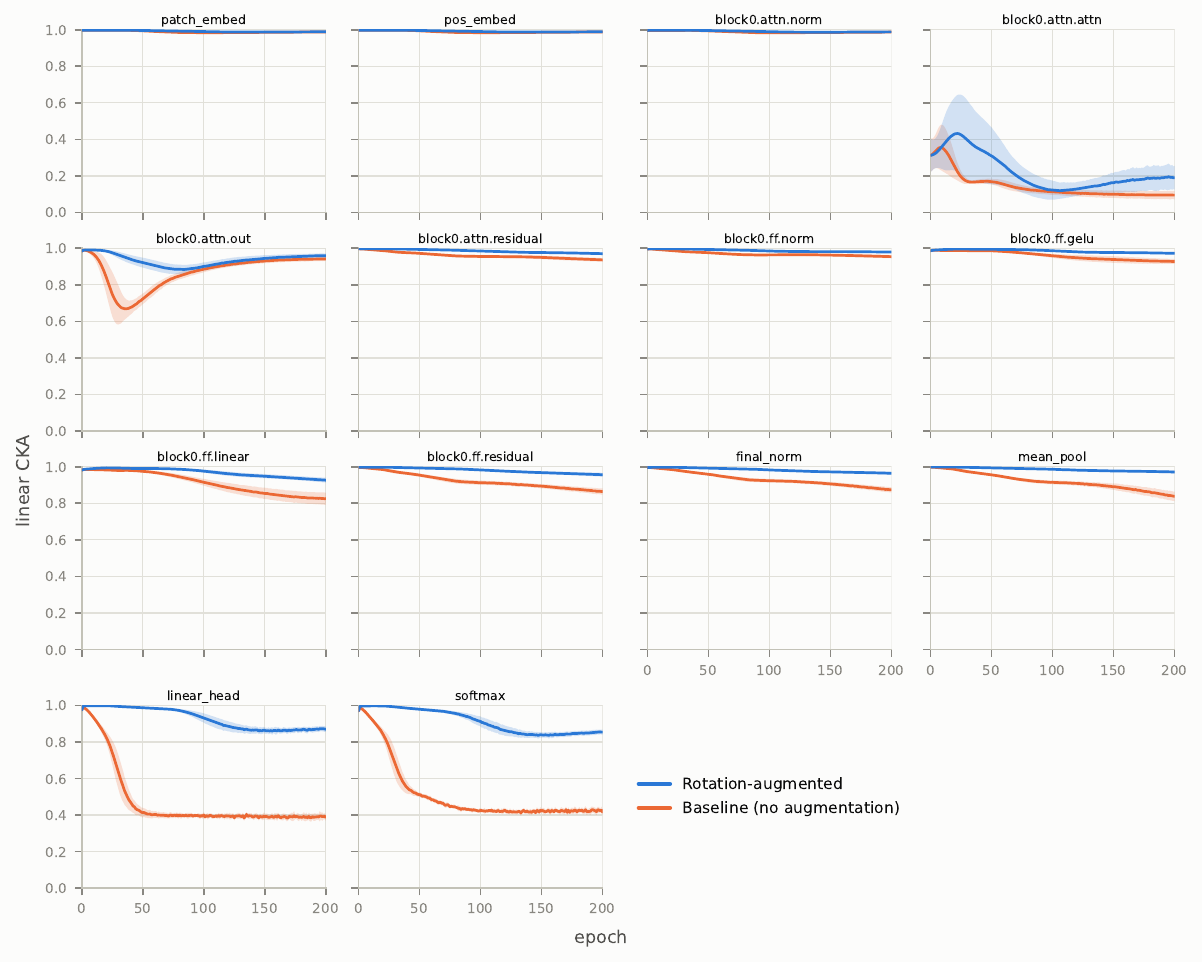}
    \caption{Full view of a ViT progression on learning equivariance on CIFAR-10 classification.}
    \label{fig:classification_vit_cifar}
\end{figure}

\begin{figure}[h]
    \centering
    \includegraphics[width=\linewidth]{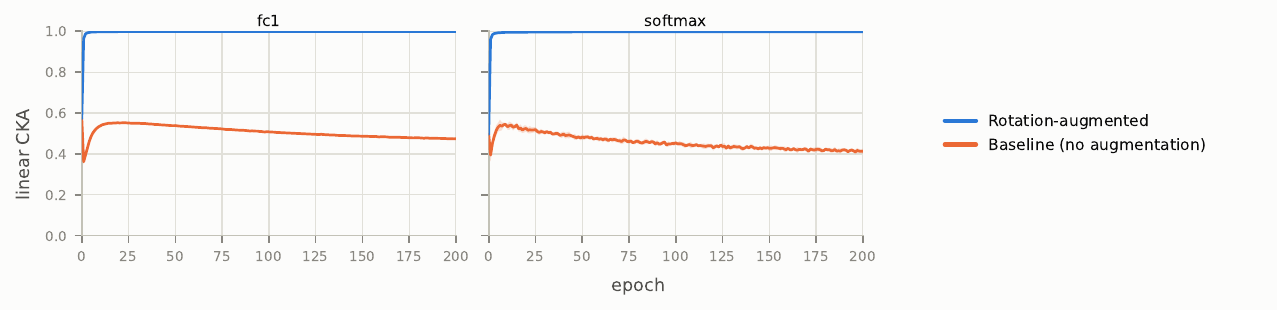}
    \caption{Full view of a perceptron progression on learning equivariance on CIFAR-10 classification.}
    \label{fig:classification_naive_cifar}
\end{figure}

\begin{figure}[h]
    \centering
    \includegraphics[width=\linewidth]{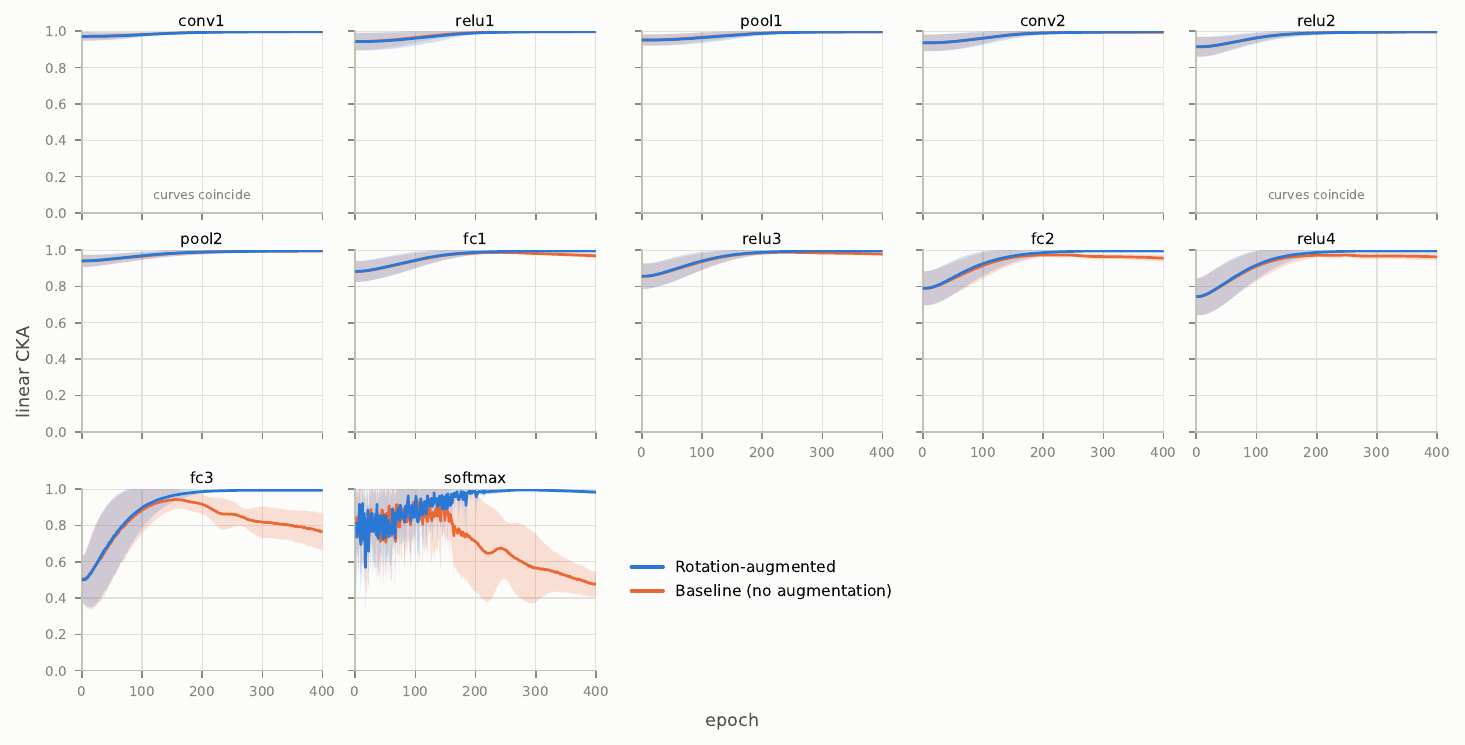}
    \caption{Full view of a CNN progression on learning equivariance on MNIST-Font classification.}
    \label{fig:classification_cnn_mnist_font}
\end{figure}

\begin{figure}[h]
    \centering
    \includegraphics[width=\linewidth]{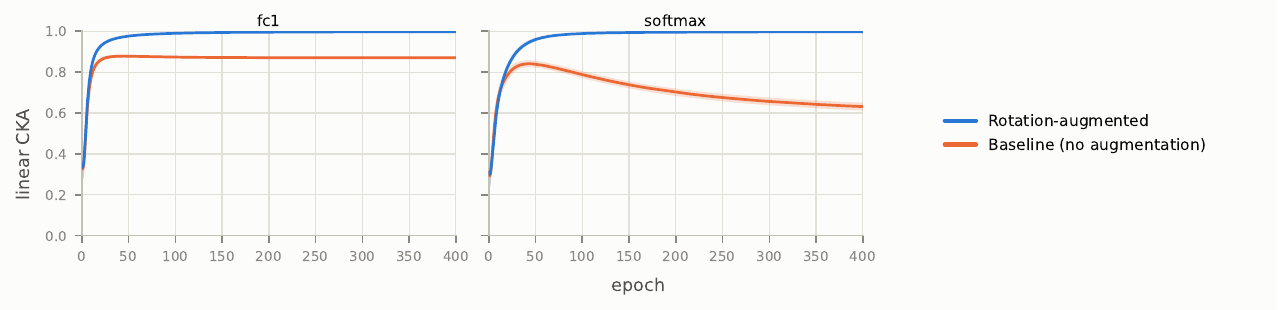}
    \caption{Full view of a perceptron progression on learning equivariance on MNIST-Font classification.}
    \label{fig:classification_naive_mnist_font}
\end{figure}

\begin{figure}[h]
    \centering
    \includegraphics[width=\linewidth]{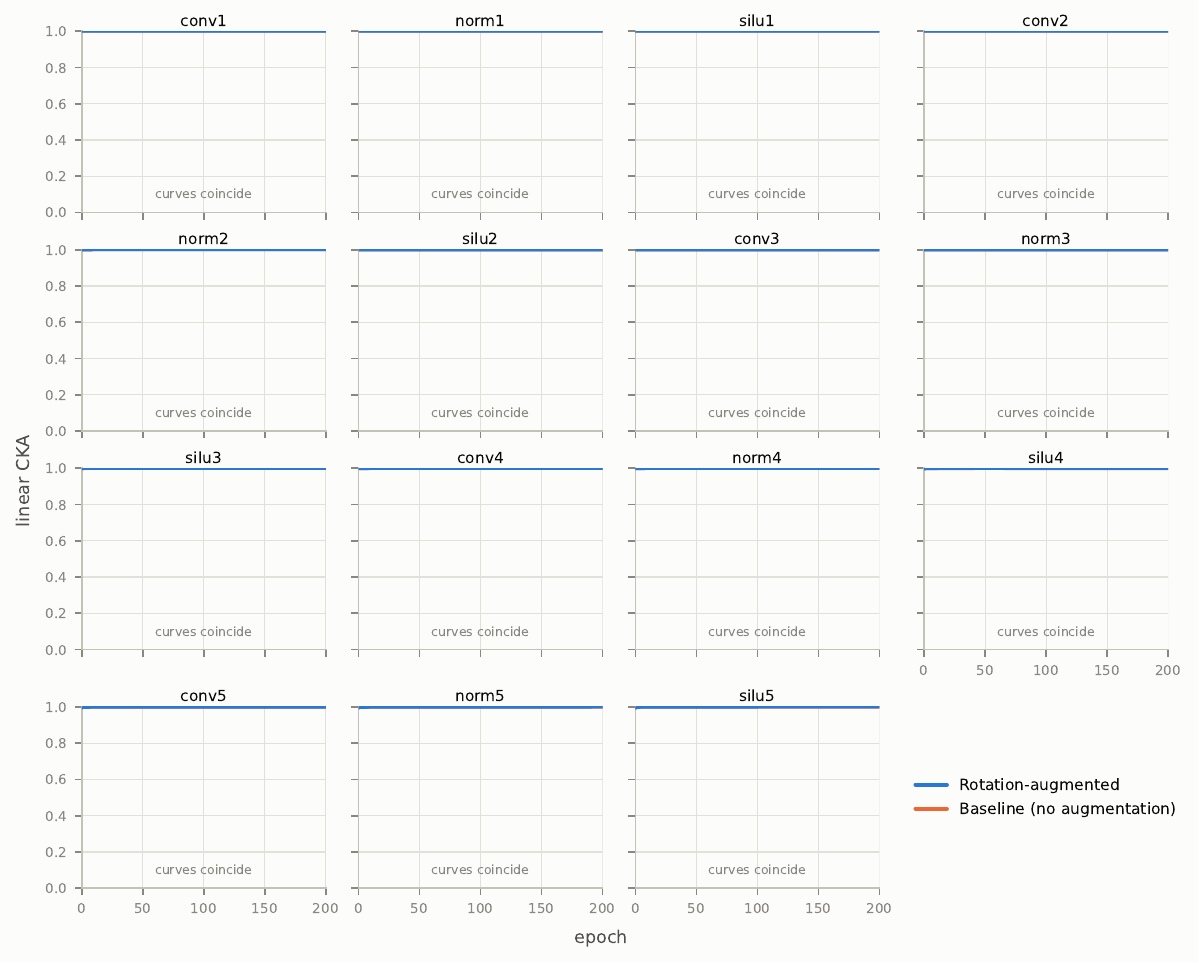}
    \caption{Full view of a CNN learning equivariance on CIFAR-10 colorization.}
    \label{fig:colorization_cnn_cifar}
\end{figure}

\begin{figure}[h]
    \centering
    \includegraphics[width=\linewidth]{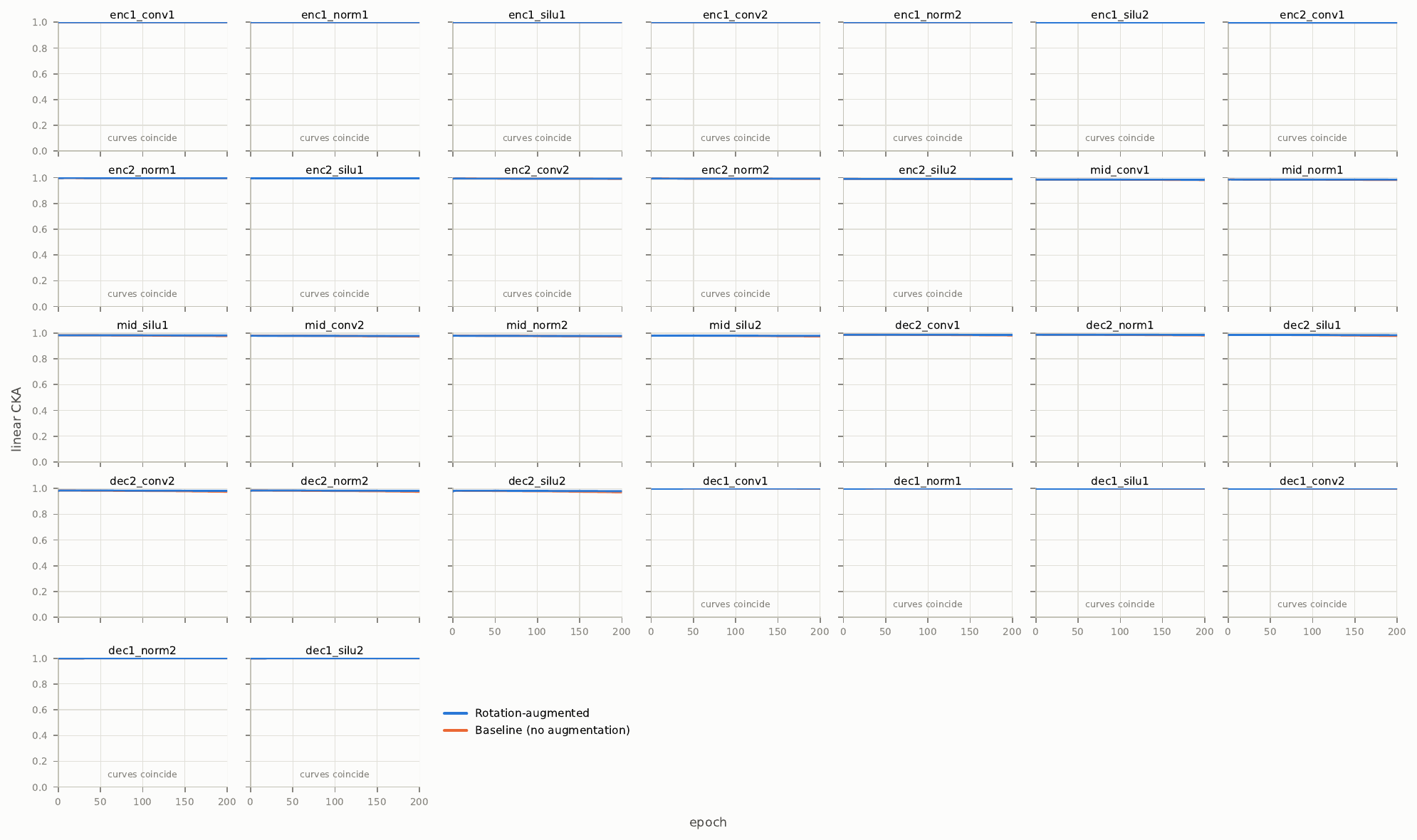}
    \caption{Full view of a UNet learning equivariance on CIFAR-10 colorization.}
    \label{fig:colorization_unet_cifar}
\end{figure}

\begin{figure}[h]
    \centering
    \includegraphics[width=\linewidth]{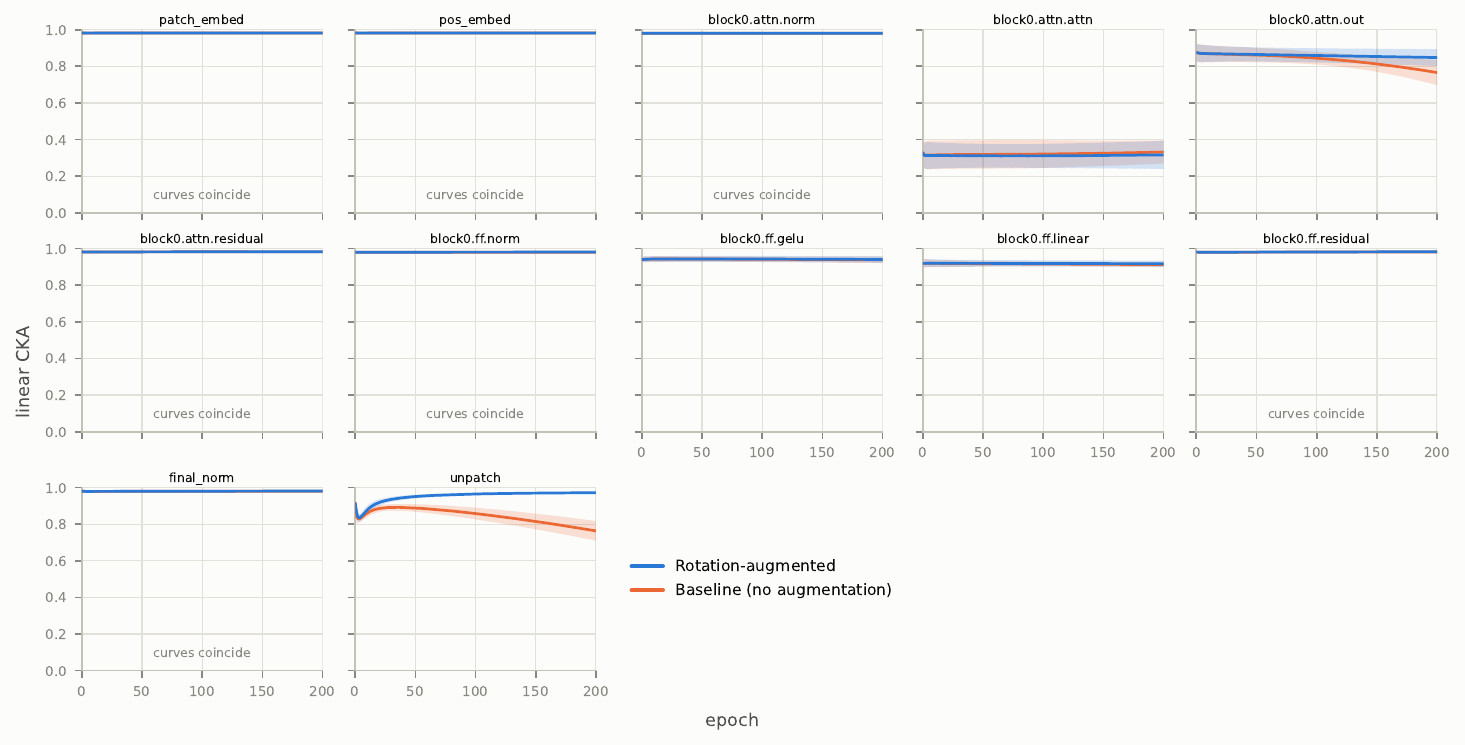}
    \caption{Full view of a ViT learning equivariance on CIFAR-10 colorization.}
    \label{fig:colorization_vit_cifar}
\end{figure}

\begin{figure}[h]
    \centering
    \includegraphics[width=\linewidth]{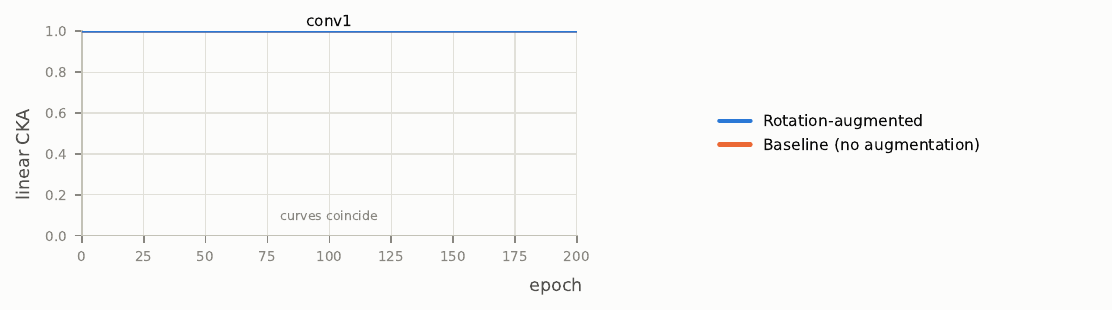}
    \caption{Full view of a perceptron learning equivariance on CIFAR-10 colorization.}
    \label{fig:colorization_naive_cifar}
\end{figure}

\begin{figure}[h]
    \centering
    \includegraphics[width=\linewidth]{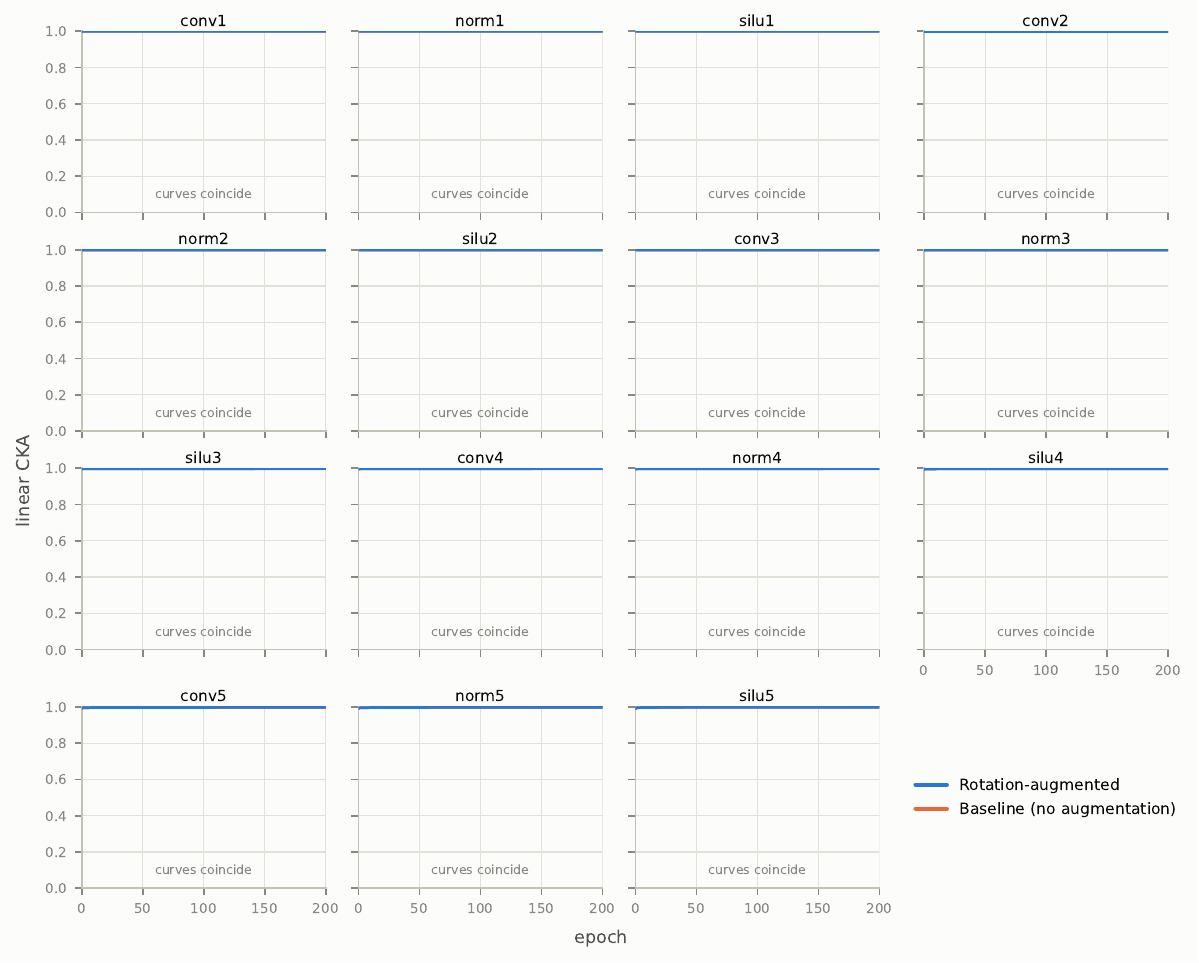}
    \caption{Full view of a CNN learning equivariance on STL-10 colorization.}
    \label{fig:colorization_cnn_stl}
\end{figure}

\begin{figure}[h]
    \centering
    \includegraphics[width=\linewidth]{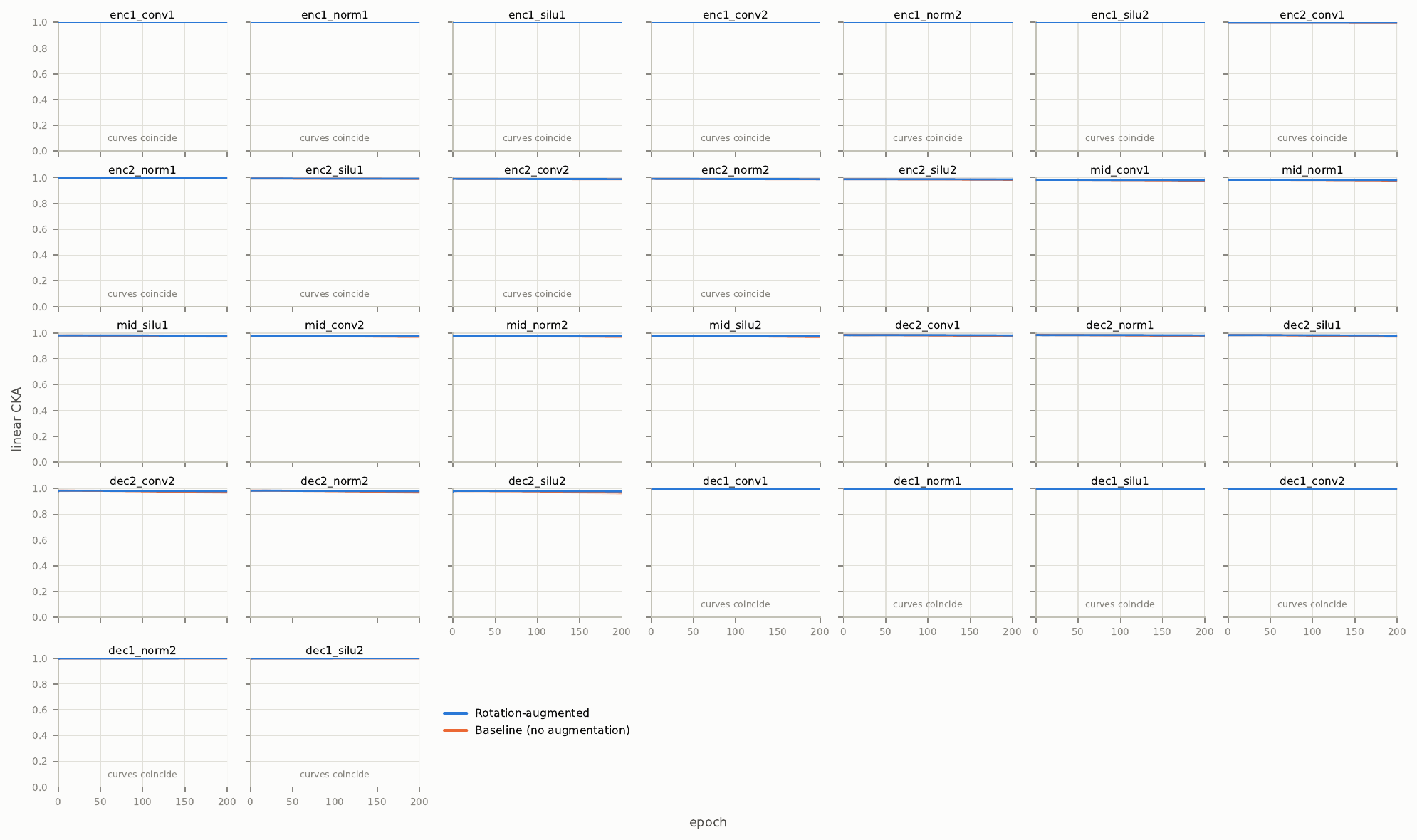}
    \caption{Full view of a UNet learning equivariance on STL-10 colorization.}
    \label{fig:colorization_unet_stl}
\end{figure}

\begin{figure}[h]
    \centering
    \includegraphics[width=\linewidth]{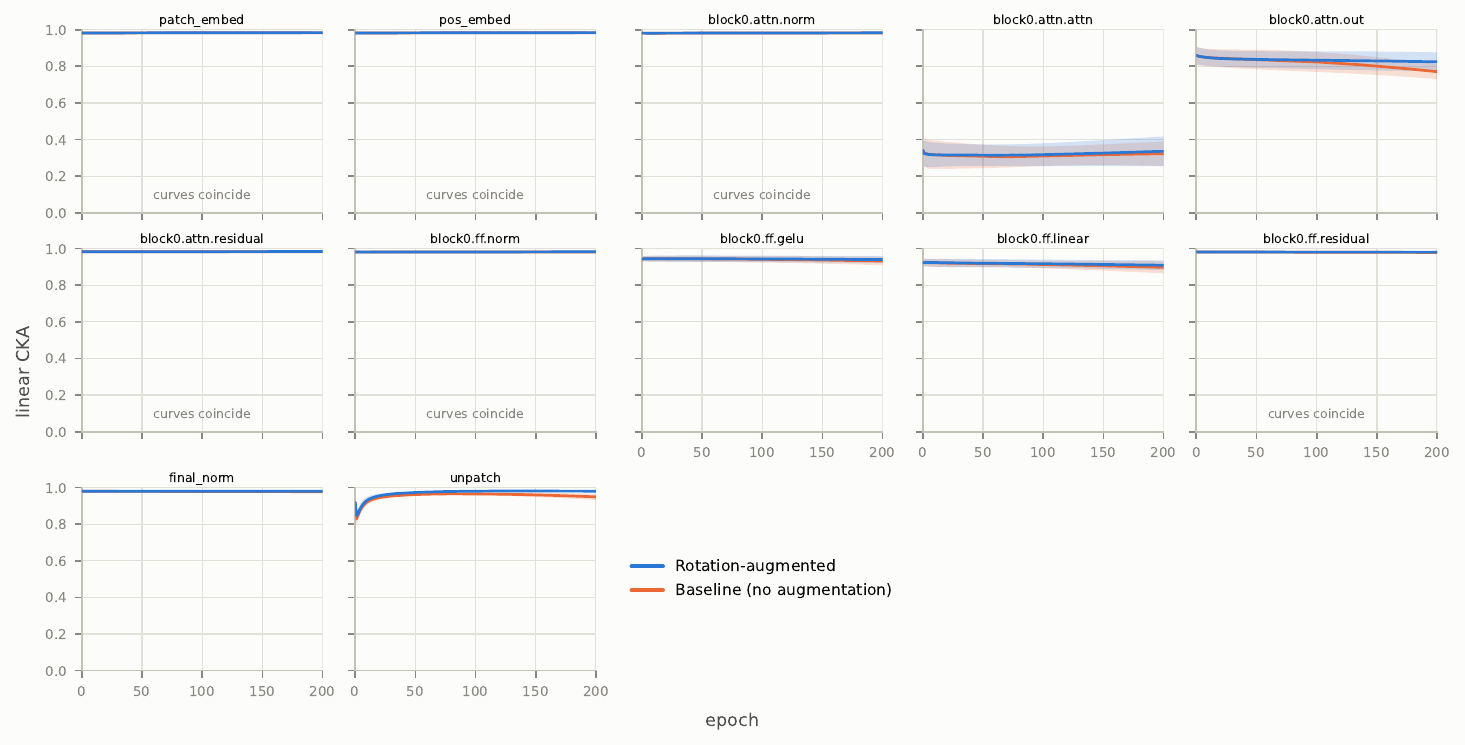}
    \caption{Full view of a ViT learning equivariance on STL-10 colorization.}
    \label{fig:colorization_vit_stl}
\end{figure}

\begin{figure}[h]
    \centering
    \includegraphics[width=\linewidth]{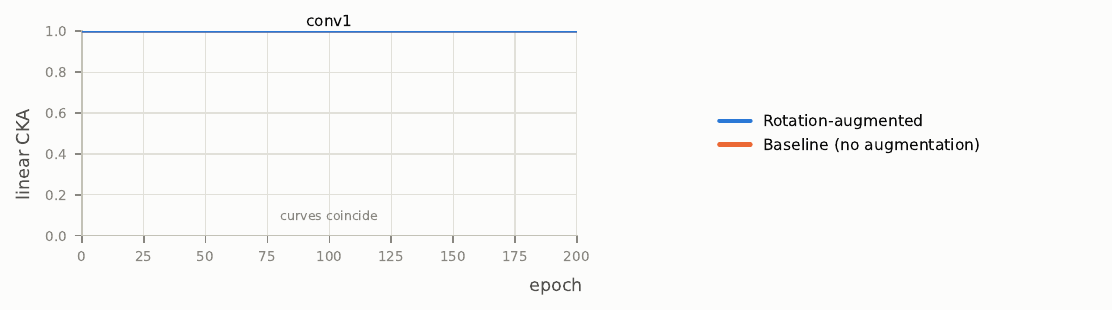}
    \caption{Full view of a perceptron learning equivariance on STL-10 colorization.}
    \label{fig:colorization_naive_stl}
\end{figure}

\begin{figure}[h]
    \centering
    \includegraphics[width=\linewidth]{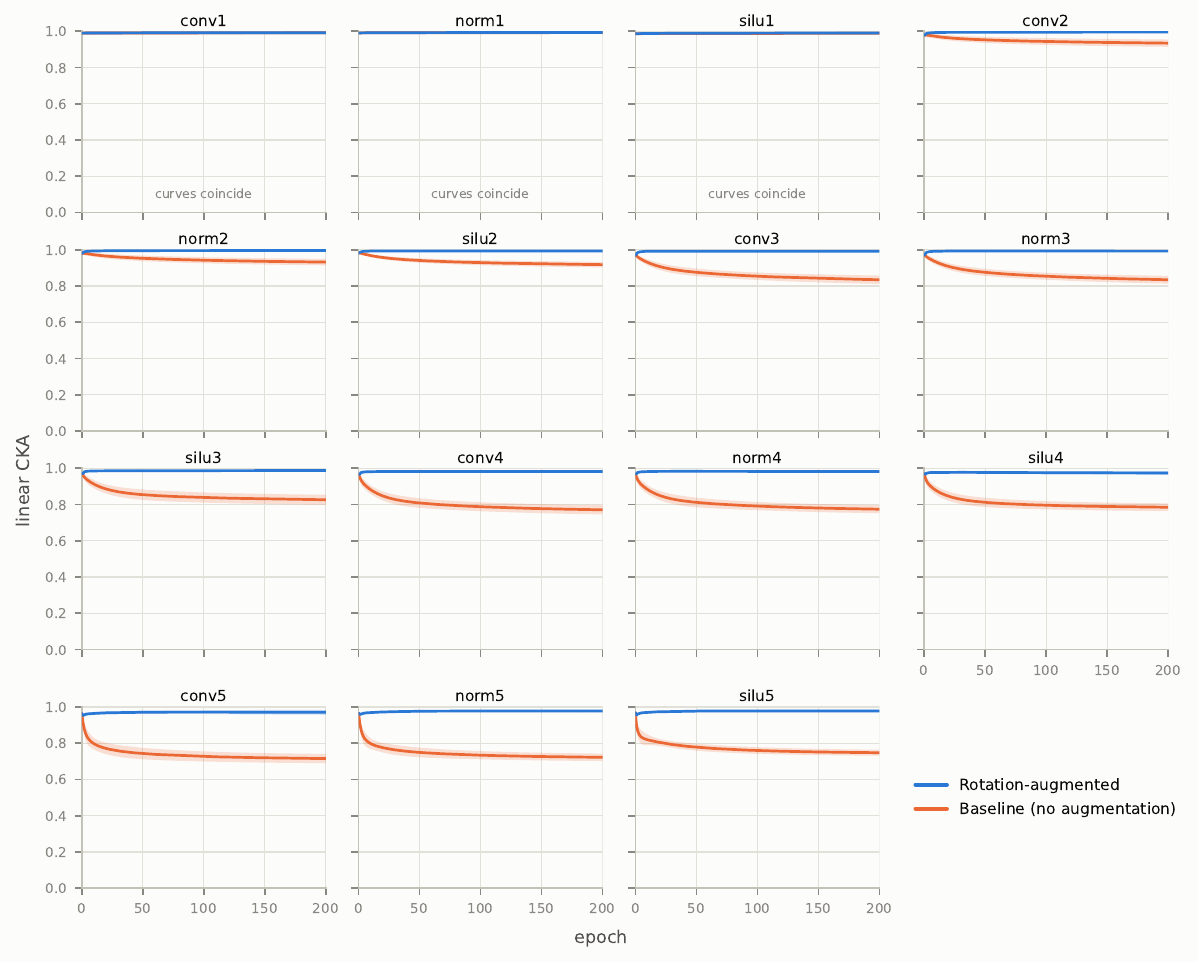}
    \caption{Full view of a CNN learning fluid dynamics on a grid-like buoyancy dataset.}
    \label{fig:fluid_flow_cnn}
\end{figure}

\begin{figure}[h]
    \centering
    \includegraphics[width=\linewidth]{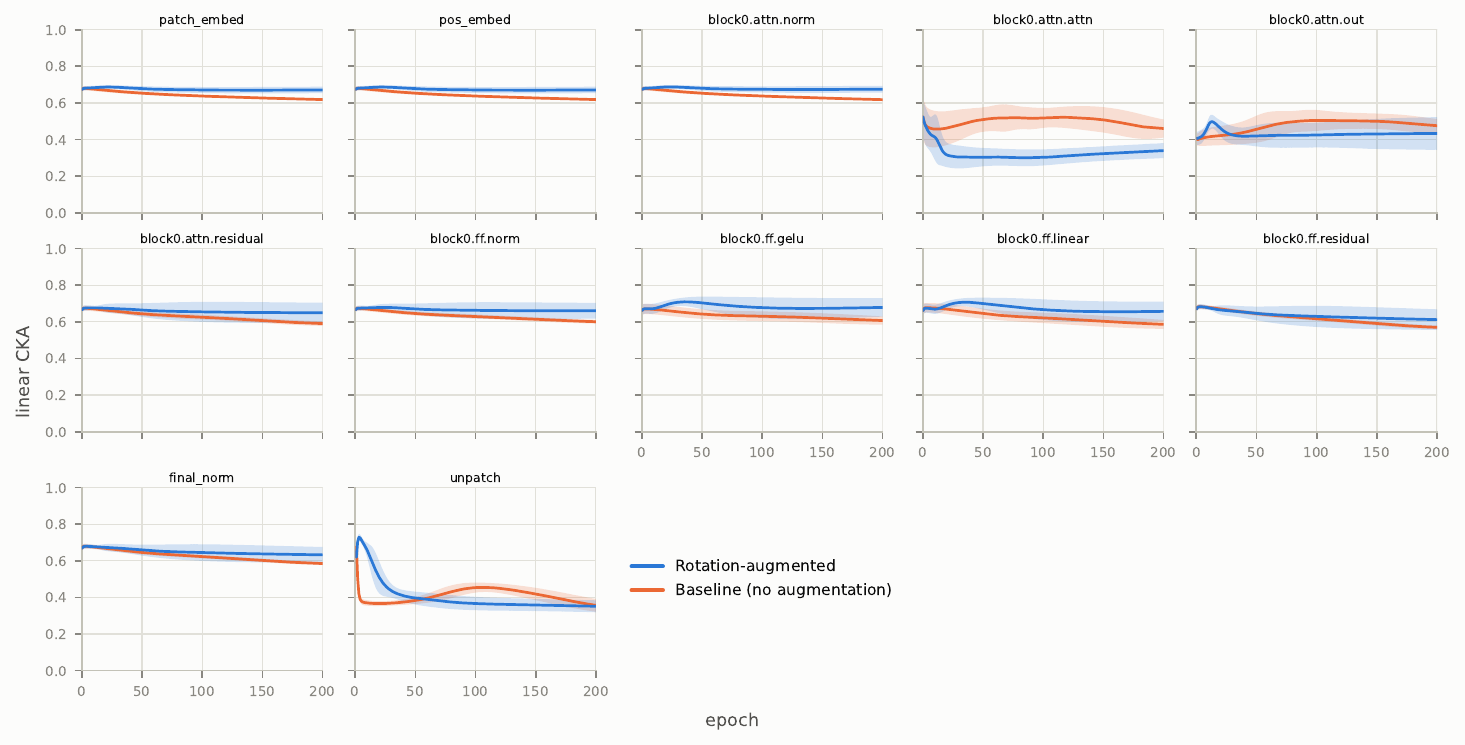}
    \caption{Full view of a ViT learning fluid dynamics on a grid-like buoyancy dataset.}
    \label{fig:fluid_flow_vit}
\end{figure}

\begin{figure}[h]
    \centering
    \includegraphics[width=\linewidth]{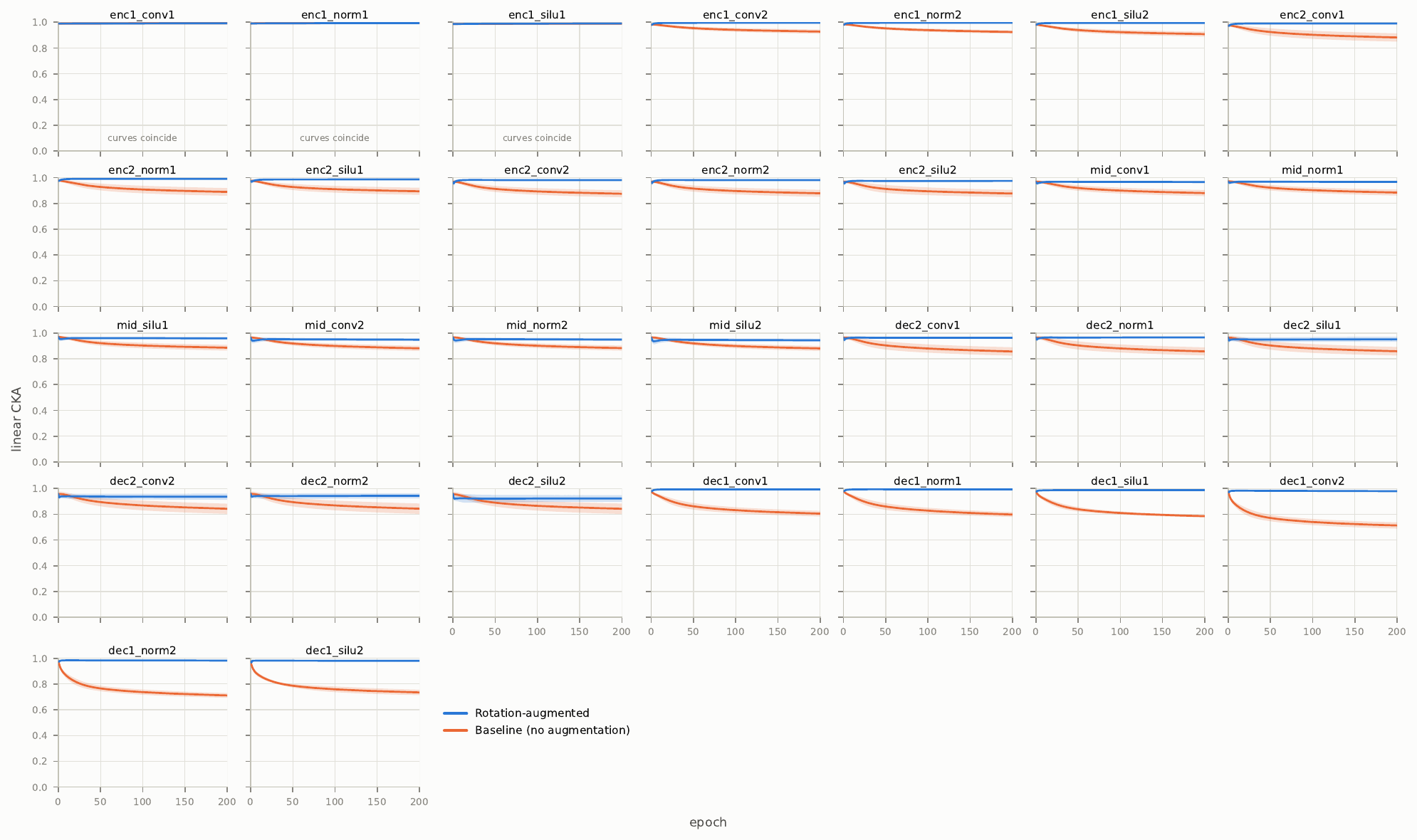}
    \caption{Full view of a UNet learning fluid dynamics on a grid-like buoyancy dataset.}
    \label{fig:fluid_flow_unet}
\end{figure}

\begin{figure}[h]
    \centering
    \includegraphics[width=\linewidth]{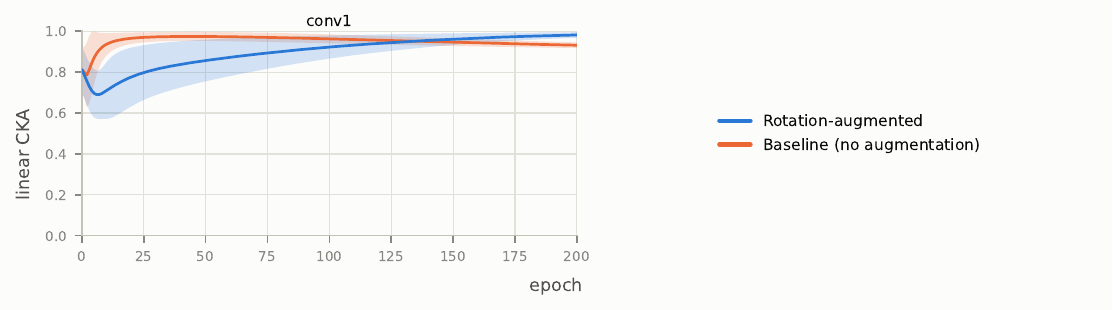}
    \caption{Full view of a perceptron learning fluid dynamics on a grid-like buoyancy dataset.}
    \label{fig:fluid_flow_naive}
\end{figure}

\begin{figure}[h]
    \centering
    \includegraphics[width=\linewidth]{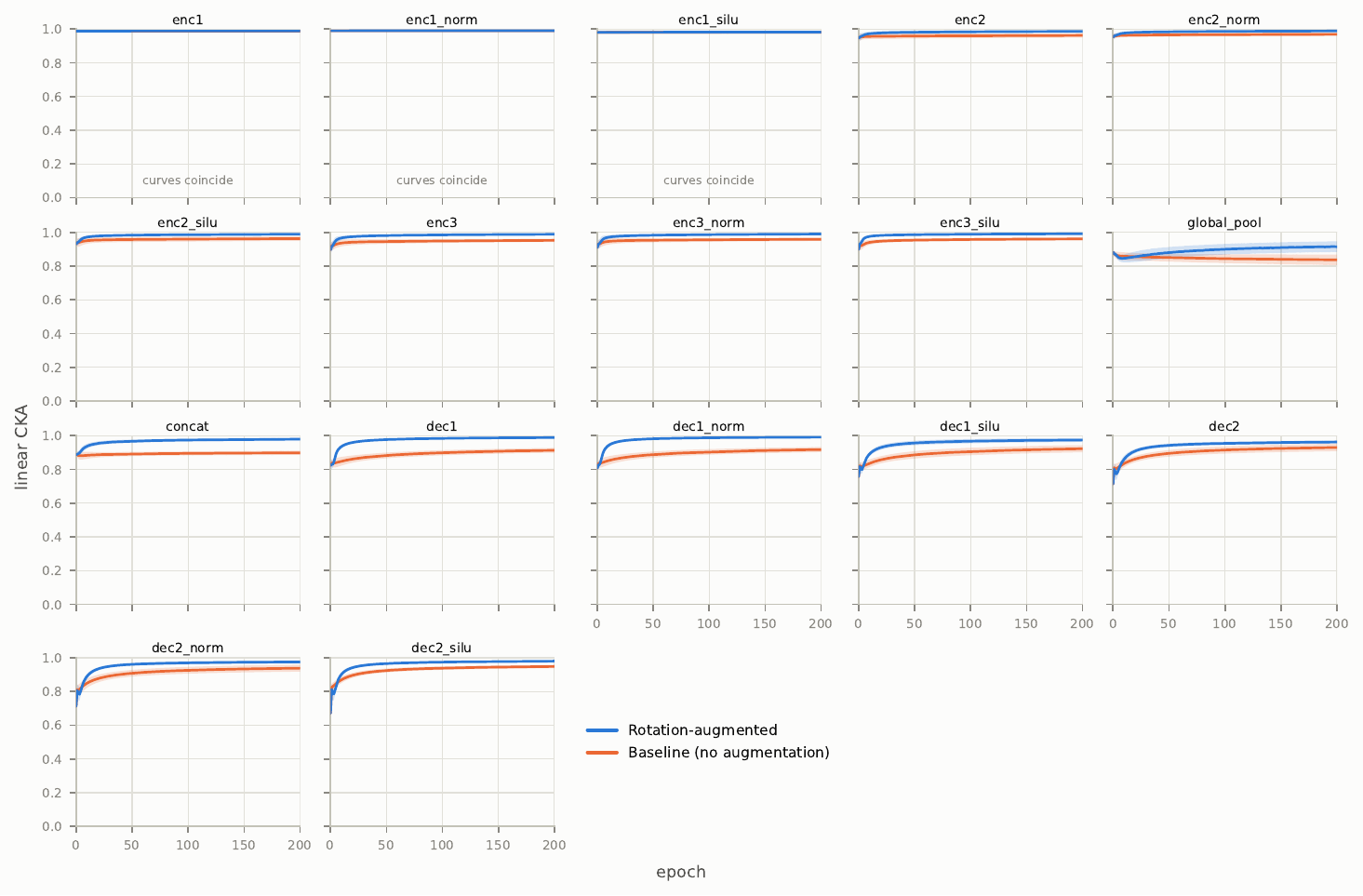}
    \caption{Full view of a PointNet learning fluid dynamics on a particle-like buoyancy dataset.}
    \label{fig:fluid_flow_particles_pointnet}
\end{figure}

\begin{figure}[h]
    \centering
    \includegraphics[width=\linewidth]{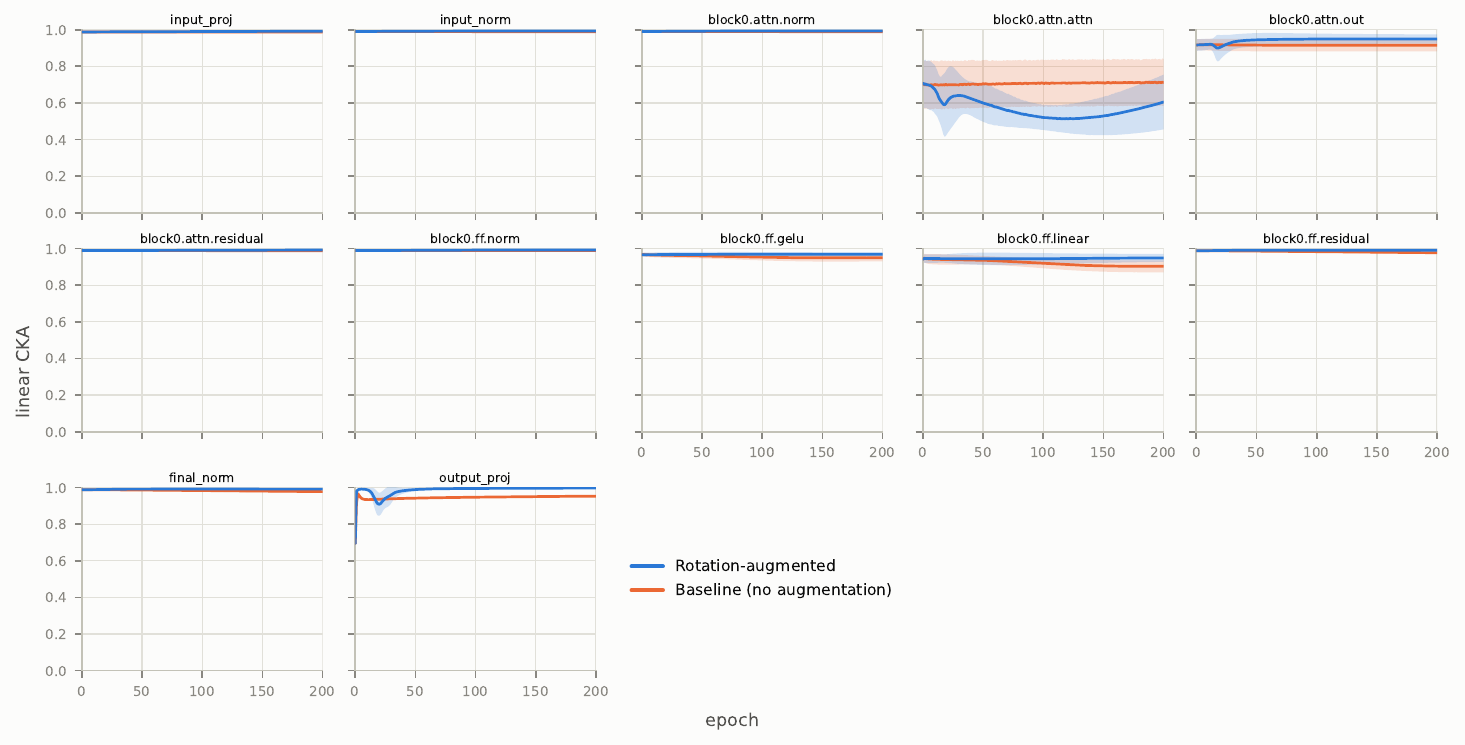}
    \caption{Full view of a SetTransformer learning fluid dynamics on a particle-like buoyancy dataset.}
    \label{fig:fluid_flow_particles_settransformer}
\end{figure}

\begin{figure}[h]
    \centering
    \includegraphics[width=\linewidth]{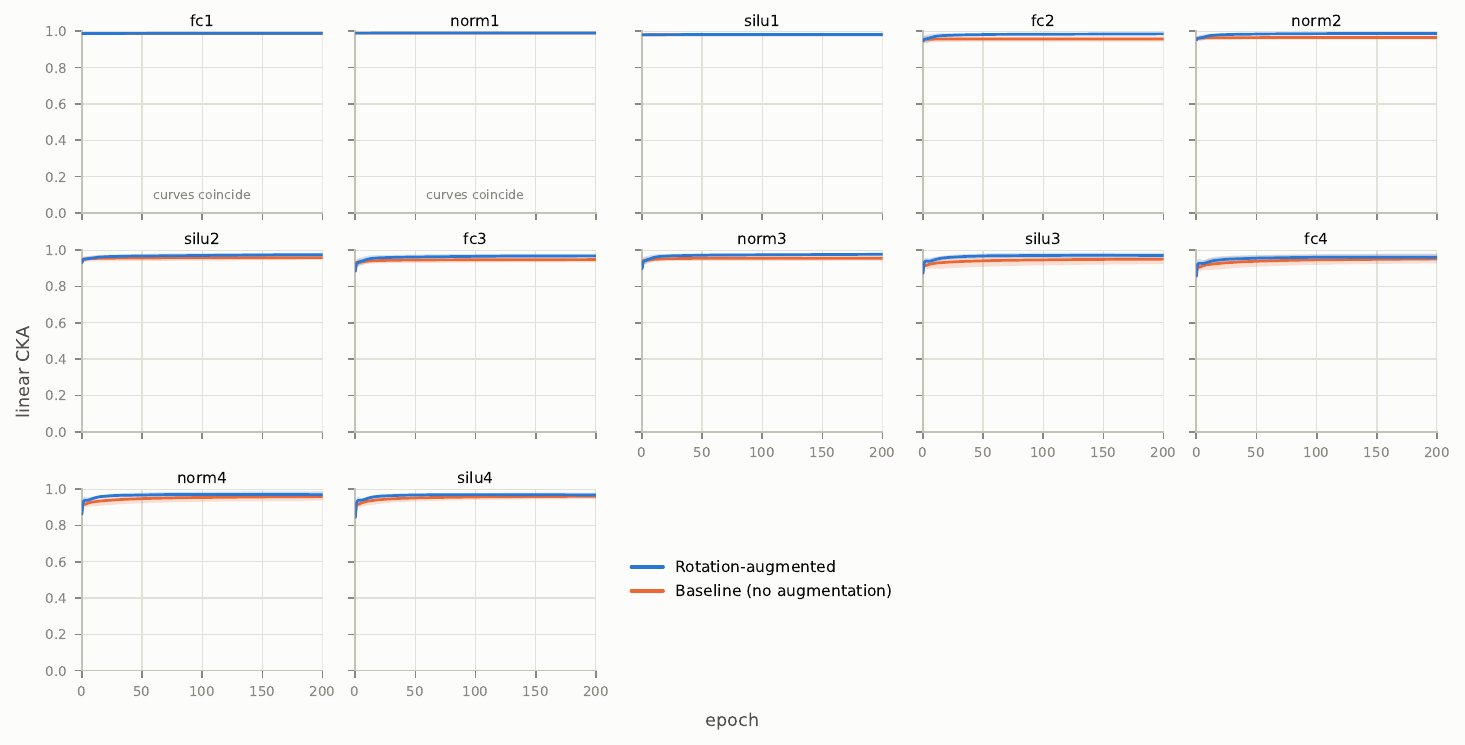}
    \caption{Full view of an MLP learning fluid dynamics on a particle-like buoyancy dataset.}
    \label{fig:fluid_flow_particles_mlp}
\end{figure}

\begin{figure}[h]
    \centering
    \includegraphics[width=\linewidth]{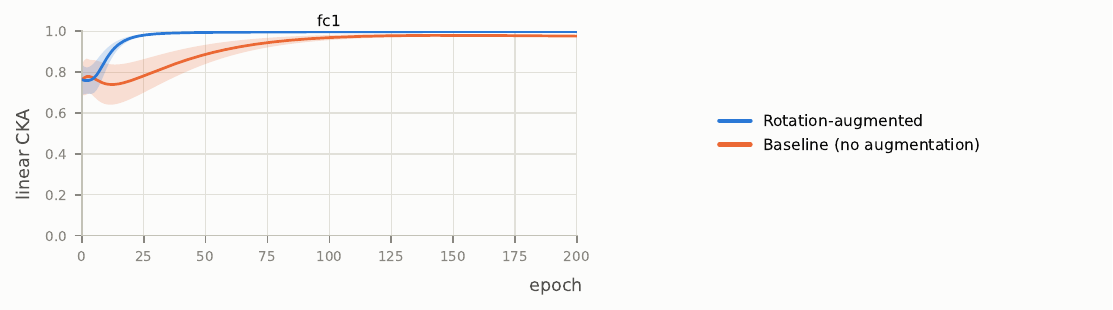}
    \caption{Full view of a perceptron learning fluid dynamics on a particle-like buoyancy dataset.}
    \label{fig:fluid_flow_particles_naive}
\end{figure}

\begin{figure}[h]
    \centering
    \includegraphics[width=\linewidth]{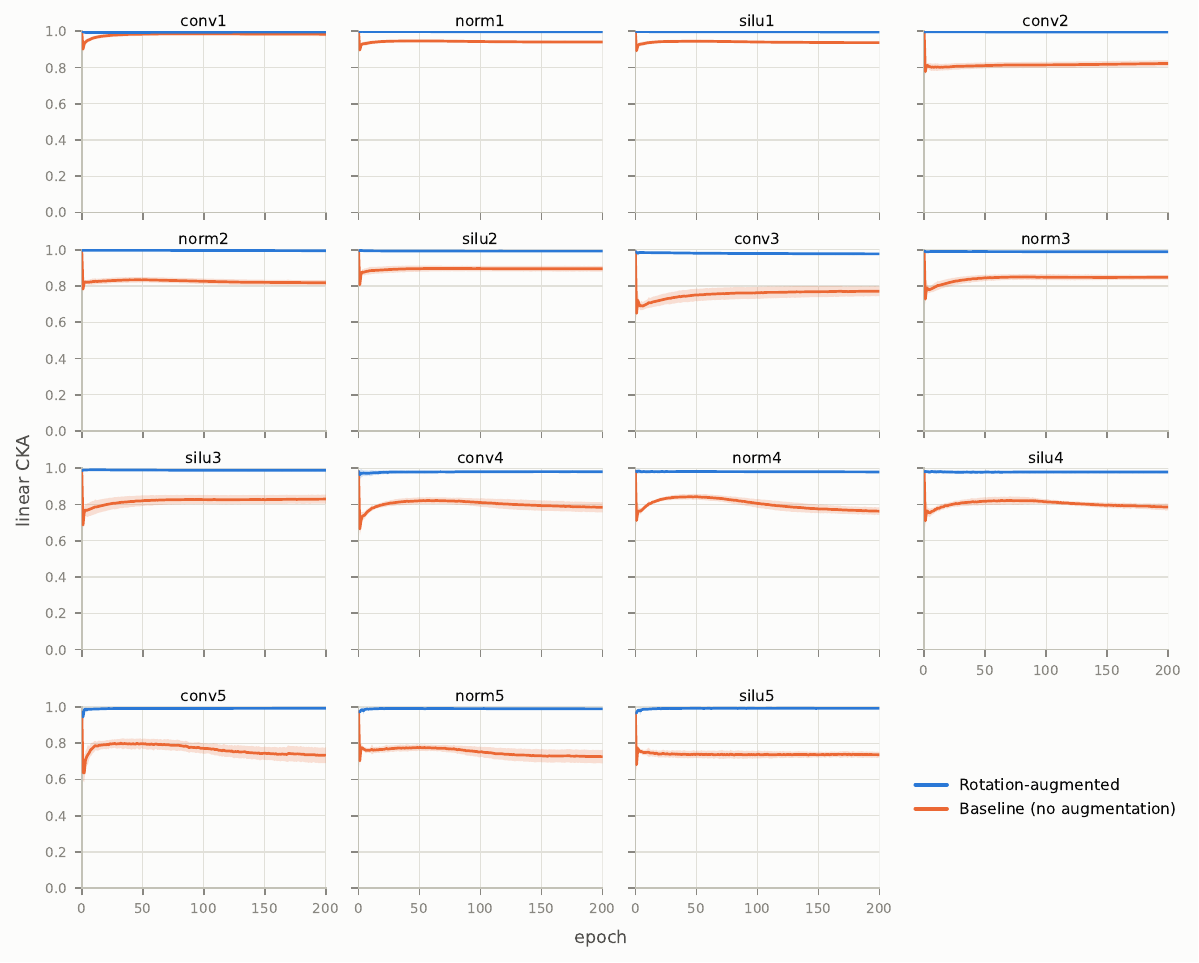}
    \caption{Full view of a CNN learning stress tensor prediction on a grid-like anisotropic dataset.}
    \label{fig:stress_prediction_cnn}
\end{figure}

\begin{figure}[h]
    \centering
    \includegraphics[width=\linewidth]{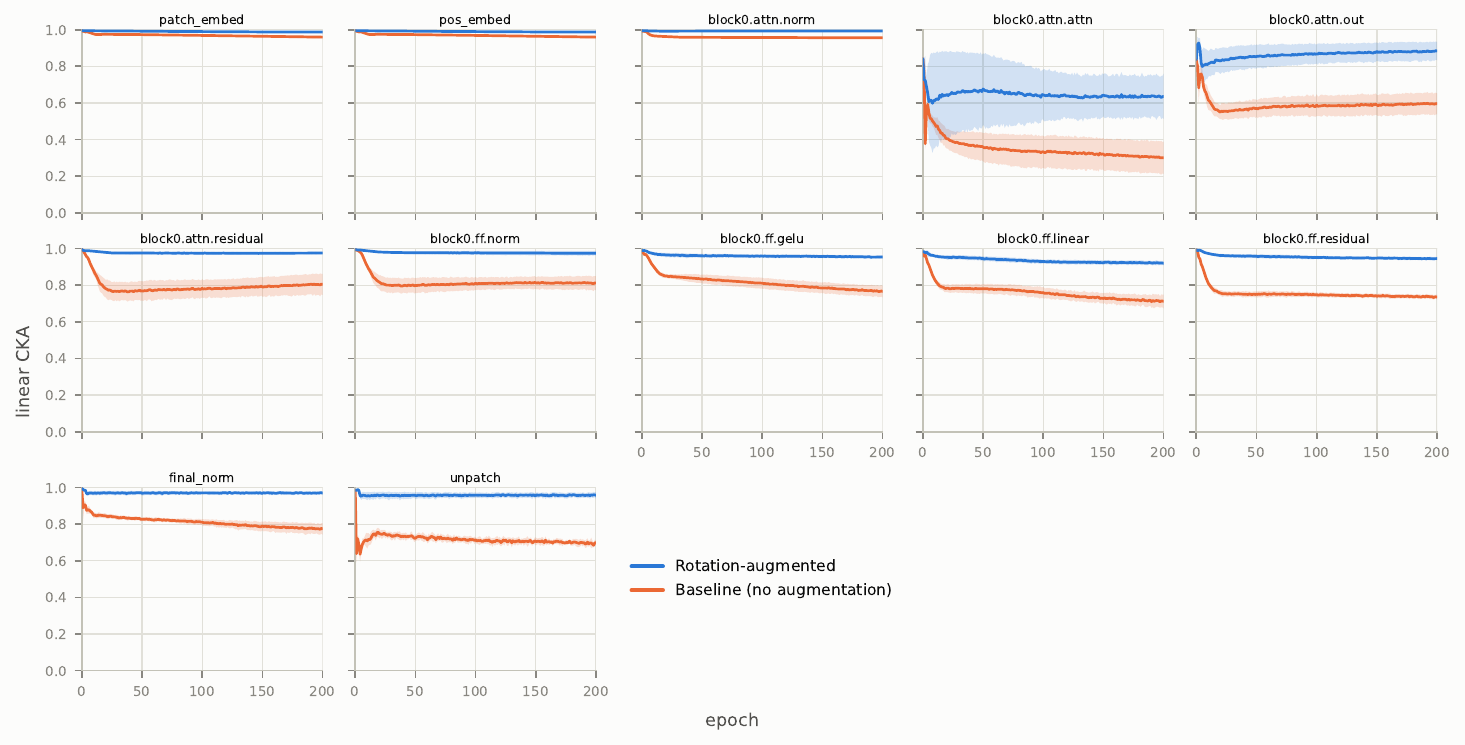}
    \caption{Full view of a ViT learning stress tensor prediction on a grid-like anisotropic dataset.}
    \label{fig:stress_prediction_vit}
\end{figure}

\begin{figure}[h]
    \centering
    \includegraphics[width=\linewidth]{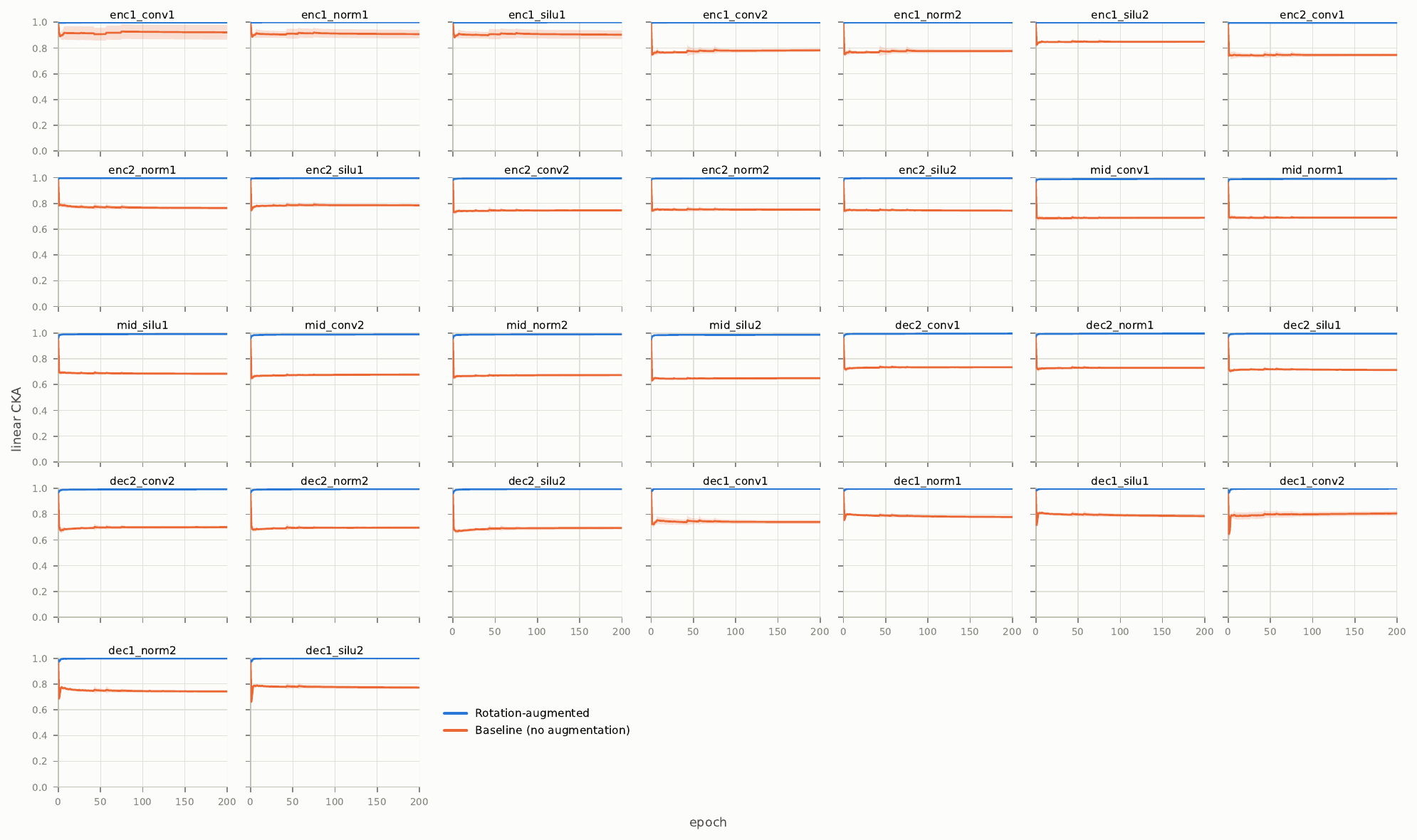}
    \caption{Full view of a UNet learning stress tensor prediction on a grid-like anisotropic dataset.}
    \label{fig:stress_prediction_unet}
\end{figure}

\begin{figure}[h]
    \centering
    \includegraphics[width=\linewidth]{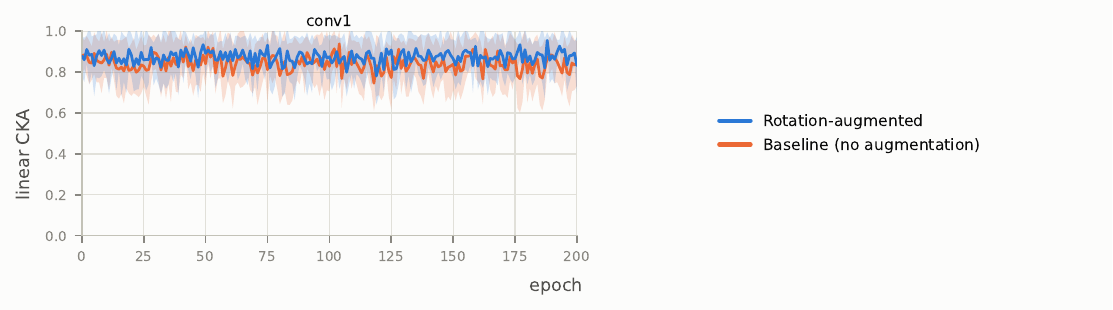}
    \caption{Full view of a perceptron learning stress tensor prediction on a grid-like anisotropic dataset.}
    \label{fig:stress_prediction_naive}
\end{figure}

\begin{figure}[h]
    \centering
    \includegraphics[width=\linewidth]{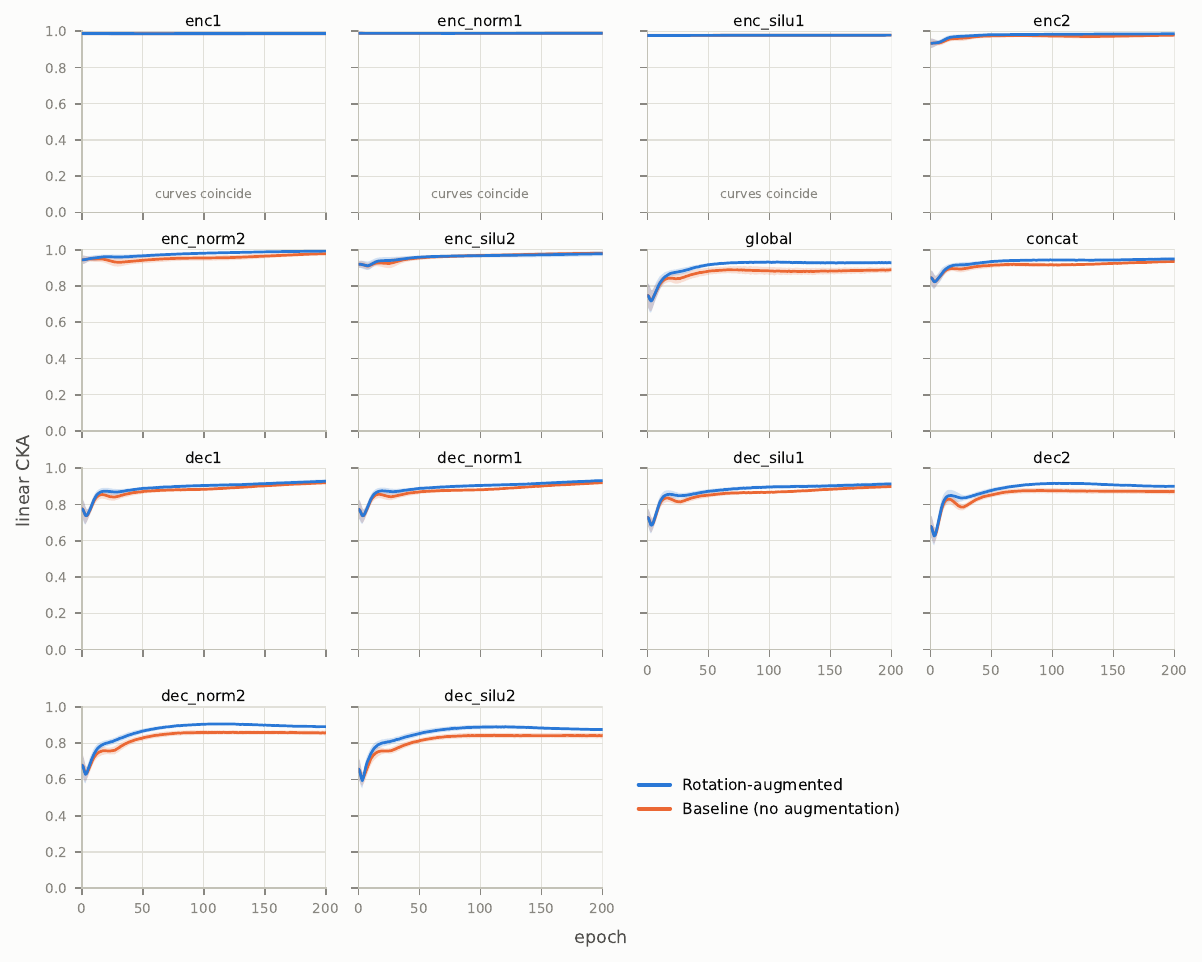}
    \caption{Full view of a PointNet learning stress tensor prediction on a particle-like anisotropic dataset.}
    \label{fig:stress_prediction_particles_pointnet}
\end{figure}

\begin{figure}[h]
    \centering
    \includegraphics[width=\linewidth]{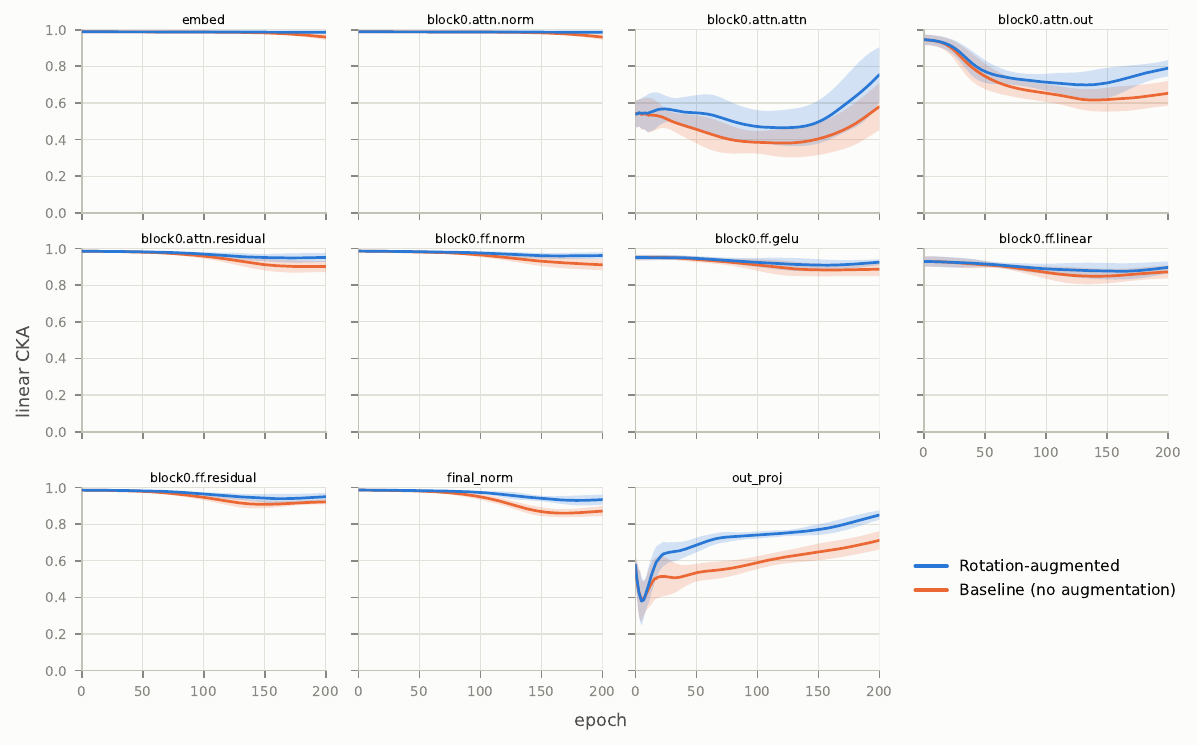}
    \caption{Full view of a SetTransformer learning stress tensor prediction on a particle-like anisotropic dataset.}
    \label{fig:stress_prediction_particles_settransformer}
\end{figure}

\begin{figure}[h]
    \centering
    \includegraphics[width=\linewidth]{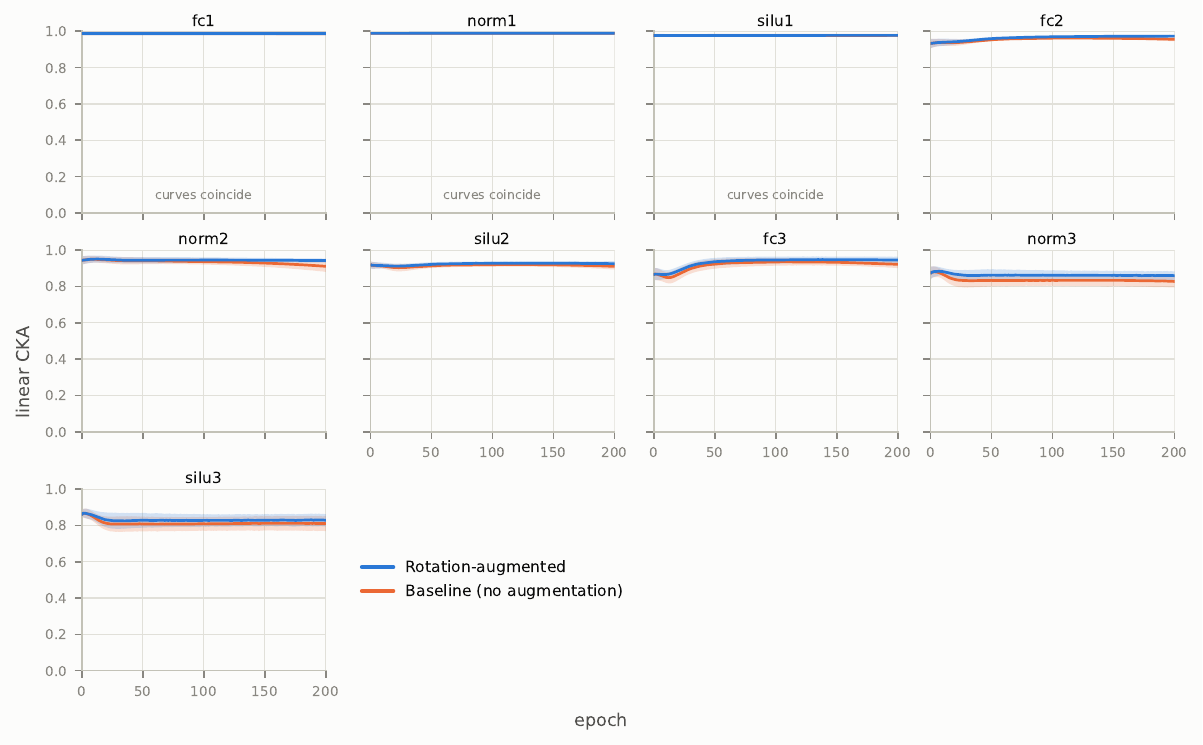}
    \caption{Full view of an MLP learning stress tensor prediction on a particle-like anisotropic dataset.}
    \label{fig:stress_prediction_particles_mlp}
\end{figure}

\begin{figure}[h]
    \centering
    \includegraphics[width=\linewidth]{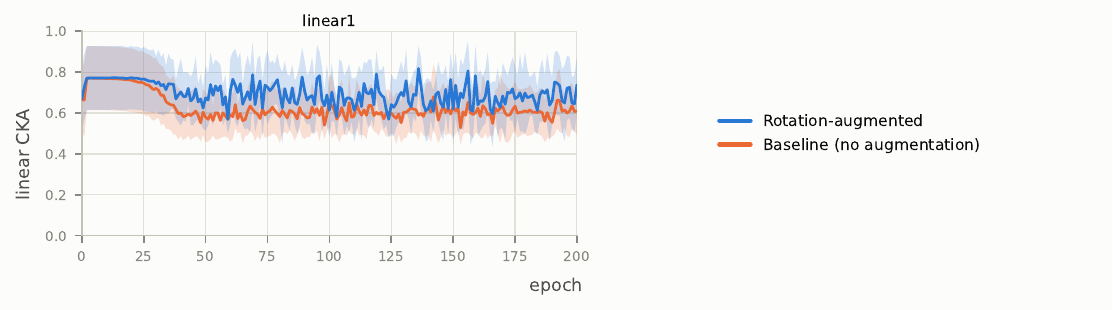}
    \caption{Full view of a perceptron learning stress tensor prediction on a particle-like anisotropic dataset.}
    \label{fig:stress_prediction_particles_naive}
\end{figure}

\end{document}

%% file: math_commands.tex
\usepackage{amsmath,amsfonts,bm}

\def\eqref#1{equation~\ref{#1}}
\def\1{\bm{1}}

\DeclareMathAlphabet{\mathsfit}{\encodingdefault}{\sfdefault}{m}{sl}
\SetMathAlphabet{\mathsfit}{bold}{\encodingdefault}{\sfdefault}{bx}{n}